\documentclass[12pt,A4]{article}
\usepackage[margin=1in,footskip=0.25in]{geometry}

\usepackage[comma]{natbib}
\usepackage{mathtools}
\usepackage{afterpage}
\usepackage{caption} 
\usepackage[utf8]{inputenc} 
\usepackage[T1]{fontenc}    
\usepackage{hyperref}       
\usepackage{url}            
\usepackage{booktabs}       
\usepackage{amsfonts}       
\usepackage{nicefrac}       
\usepackage{microtype}      
\usepackage{amsthm}
\usepackage{graphicx}
\usepackage{bm,color,paralist}
\usepackage{algorithm,algpseudocode}
\usepackage{times}
\usepackage{subfig,multirow}
\usepackage{enumitem}
\usepackage{epstopdf}
\usepackage{setspace}
\usepackage{hyperref}
\usepackage{mathrsfs}

\newcommand{\X}{\mathbf{X}}
\newcommand{\D}{\mathbf{D}}
\newcommand{\cc}{\mathbf{c}}
\newcommand{\dd}{\mathbf{d}}

\usepackage{amsmath}
\usepackage{subfig}

\newcommand{\blank}[1]{}
\newcommand{\ie}{{i.e.}}

\newcommand*{\Scale}[2][4]{\scalebox{#1}{$#2$}}
\newtheorem{theorem}{Theorem}

\newtheorem{lemma}{Lemma}

\newtheorem{problem}{Problem}
\newtheorem{assumption}{Assumption}

\begin{document}
    
    \title{\bf Balance-Subsampled Stable Prediction}

    \author{Kun Kuang$^1$$^*$, Hengtao Zhang$^2$$^*$, Fei Wu$^1$, \\
    Yueting Zhuang$^1$, and Aijun Zhang$^2$$^{**}$\\
        {\normalsize $^1$College of Computer Science and Technology, Zhejiang University, Hangzhou, China}\\
        {\normalsize $^2$Department of Statistics and Actuarial Science, The University of Hong Kong,}\\
        {\normalsize Hong Kong, China}}
    
    \date{}
    \maketitle
    \def\thefootnote{*}\footnotetext{These authors contributed equally to this work}\def\thefootnote{\arabic{footnote}}
    \def\thefootnote{**}\footnotetext{Corresponding author: ajzhang@hku.hk}\def\thefootnote{\arabic{footnote}}

    \begin{abstract}
   In machine learning, it is commonly assumed that training and test data share the same population distribution. However, this assumption is often violated in practice because the sample selection bias may induce the distribution shift from training data to test data. 
   Such a model-agnostic distribution shift usually leads to prediction instability across unknown test data. In this paper, we propose a novel  balance-subsampled stable prediction (BSSP) algorithm based on the theory of fractional factorial design. 
   It isolates the clear effect of each predictor from the confounding variables. 
   A design-theoretic analysis shows that  the proposed method can reduce the confounding effects among predictors induced by the distribution shift, hence improve both the accuracy of parameter estimation and prediction stability. 
   Numerical experiments on both synthetic and real-world data sets demonstrate that our  BSSP algorithm significantly outperforms the baseline methods for stable prediction across unknown test data.
    \vskip 6.5pt \noindent {\bf Keywords}: 
    Stable Prediction; Distribution Shift; Fractional Factorial Design; Subsampling; Regression; Classification.
    \end{abstract}
    
    \section{Introduction}
    One of the most common assumptions in learning algorithms is the homogeneity among training and test samples on which the algorithm is expected to make predictions. However, this condition is often violated in practice due to sample selection bias, which causes distribution shifts between observed data and the population.  
    Moreover, the unknown test distribution leads to an agnostic distribution shift problem.
    Therefore, it is highly demanding
    to develop predictive algorithms that are robust to
    the agnostic distribution shift between training and
    unknown test data.
    
    \blank{
    Recently, the invariant learning algorithms
    have been proposed to address the agnostic distribution shift problem, including domain generalization \cite{muandet2013domain}, invariant causal prediction \cite{peters2016causal} and causal transfer learning \cite{rojas2018invariant}.
    The motivation behind these methods is to explore the invariance, including invariant representation of data, the invariant structure between predictors and outcome variable, and causal structure across multiple training datasets.
    The performance of these methods usually depends on the diversity of multiple training datasets being considered, but they cannot address the distribution shift that does not appear in existing datasets.   
    Moreover, their training complexity grows exponentially with the dimension of the
    feature space in the worst case, which is not acceptable in practice.
    }
    
    In this paper, we assume that the underlying predictive mechanism between predictors/features $\mathbf{X}$ and outcome variable $Y$ is invariant across datasets. Based on the invariant predictive mechanism, all predictors $\mathbf{X}$ fall into one of two categories. One category includes stable features $\mathbf{S}$, which have causal effects on outcome $Y$, and are stable/invariant across datasets.
    For example in computer vision, ears, noses, and legs of dogs are stable features to recognize whether an image contains a dog or not.
    The other category includes noisy features $\mathbf{V}$, which have no causal effects on outcome, but might be highly correlated with either  stable features, outcome variable or both in certain data sets. For the same example, the grass and background pixels are noisy features for dog recognition.
    Hence, taking the regression task as an example,  we set $\mathbf{X} = \{\mathbf{S}, \mathbf{V}\}$ and have $Y = f(\mathbf{X}) + \epsilon = f(\mathbf{S})+\epsilon$ in our problem. 
    Conditional on the full set of stable features, the noisy features do not affect the expected outcome.
    However, the distribution shift might make a part of noisy features to become power predictors. In the previous example, grass would be a power predictor if most of the dogs in training data are on the grass.
    Different distributions shifts might appear in different datasets, leading to the variation of confounding and spurious correlation\footnote{Here, we call the correlation between noisy features and outcome variable as spurious correlation, since in the generation of the outcome variable, noisy features are supposed to be not correlated with the outcome.} between the noisy features and outcome variable. To address the stable prediction problem, we should reduce such confounding effects, hence removing the spurious correlations between noisy features and outcome variable.
    
    We assume no prior knowledge about which features are stable and which are noisy. Under such settings, one possible way is to isolate the impact of each individual feature for recovering the true correlation (causation) between predictors and the outcome variable. Variable balancing techniques are widely used for causation recovery in the literature of causal inference. The key idea is to construct sample weights by either employing propensity scores \citep{rosenbaum1983central,kuang2017treatment,austin2011introduction} or optimizing weights directly \citep{athey2018approximate,zubizarreta2015stable,hainmueller2012entropy,kuang2017estimating}.
    Recently, a global balancing algorithm \citep{kuang2018stable} was proposed to learn the weights that enforce all features to be as independent as possible, in order to isolate the impact of each feature.
    Despite its better performance, this algorithm only focuses on the confounding factors between any two variables, while ignoring the higher-order interactions. Moreover, it is not an efficient way by learning the weight for each sample and using full data to perform fitting, especially with huge training samples in big data scenarios.
    
    Full and fractional factorial designs are widely used in statistics for arranging factorial experiments without confounding effects \citep{box2005statistics,dey2009fractional}. Using data collected from a factorial designed experiment, one can easily isolate the impact of each feature and reveal the causation between predictors and the outcome variable. Inspired by the factorial designs of experiment, we propose a balance-subsampled stable prediction (BSSP) algorithm, which consists of a factorial design-based subsampling strategy for covariates balancing and a subsampled learning model for stable prediction. Using the factorial design, the subsampling strategy selects a subset of samples from training data such that the covariates are mutually balanced and thus deconfounded.
    Then, the model fitted by the subsamples would exploit the true correlations between predictors and outcome for stable prediction.
    Our algorithm has the overwhelming performance across unknown test data with a distribution shift from the training data, thus achieve a more stable prediction. Furthermore, we can train the model faster as the subdata is much smaller than the full data.
    
    In summary, the contributions of this paper are listed as follows:
    \begin{itemize}[leftmargin=0.7cm]
    \item We propose a balance-subsampled stable prediction (BSSP) algorithm based on a fractional factorial design-based subsampling strategy for variable deconfounding.
    \item Theoretically we show that the fractional factorial design-based subsampling can remove the confounding effects with non-linear interactions. Hence, our BSSP algorithm can precisely estimate the parameters and achieve a stable prediction across unknown environments.
    \item We conduct extensive experiments on both synthetic and real-world datasets, and demonstrate the advantages of our algorithm for stable prediction in both regression and classification tasks.
    \end{itemize}

    \section{Related Work}
    
    Recently, the invariant learning methods have been proposed to explore the invariance across multiple training datasets and used for stable prediction on unknown testing data.
    \cite{peters2016causal} proposed an algorithm to exploit the invariance of a prediction under the causal model, and identify invariant features for causal prediction.
    \cite{rojas2018invariant} proposed to learn the invariant structure between predictors and the response variable by a causal transfer framework. 
    Similarly, domain generalization methods \citep{muandet2013domain} try to discover an invariant representation of data for prediction on unknown test data.
    The main drawback of these methods is that their performance highly depends on the diversity of the multiple training data being considered. Moreover, they cannot handle well the distribution shift that does not appear in the existing training data. 
    
    \cite{kuang2018stable} proposed a global balancing algorithm for stable prediction. As shown in Eq. (\ref{eq:L_balancing}), the global balancing algorithm attempts to learn global sample weights for each sample such that all predictors may become independent. \cite{kuang2018stable} also proved that the ideal global sample weights can isolate the impact of each predictor, hence address the stable prediction problem. However, the algorithm in \cite{kuang2018stable} is non-convex and only focuses on the first-order confounding between any two variables, ignoring the higher-order interactions.
    
    In statistical designs of experiments, full and fractional factorial designs, especially the two-level factorial designs, are widely used for experimental planning and data collection; see \cite{box2005statistics,dey2009fractional} and references therein. Resolution and minimum aberration are two main criteria to evaluate the goodness of fractional factorial designs; see \cite{fries1980minimum, ma2001note,xu2001generalized, zhang2005majorization}. These works provide efficient ways of conducting experiments, but not the sample selection methods for  observational data. 
    
    Subsampling is an efficient strategy to compress the data and accelerate the machine learning algorithm. The idea of sampling is traditionally applied in the survey area and designed to estimate point statistics before observing the response \citep{steven2012sampling}. 
    Recently, it has been leveraged to accelerate the estimation of more complex models.  \cite{drineas2011faster} and \cite{ma2015statistical} used the statistical leverages as the probability to resample observations for fitting linear models.
    \cite{wang2018optimal} proposed an optimal subsampling based on the $A$-optimal design for logistic regression, which is further improved by \cite{wang2019more}. $D$-optimal design is adopted in \cite{wang2019information} to select informative subdata for the efficient linear regression. Those methods
    aim to provide a fast approximation to the model parameters estimated by a given data. What they concern about is the computational efficiency and approximation error on the given training data. Unlike them, this paper considers the idea of subsampling for the stability of model prediction across different and possibly unknown datasets. 
    
   \section{Problem and Notations}
   For a prediction problem, we let $\mathbf{X}$ and $Y$ denote the predictors and outcome variable, respectively. And we define an \textbf{environment} to be the joint distribution $P_{XY}$ of $\{\mathbf{X}, Y\}$. Let $\mathcal{E}$ denote the set of all environments, and $\mathbf{M}^{e} = \{\mathbf{X}^{e}, Y^{e}\}$ be the dataset collected from $e\in\mathcal{E}$. For simplicity, we consider the case where features have finite support. Note that finite support features can be transferred into binary ones via, e.g., dummy encoding.  Therefore, without loss of the generality, we assume features $\mathbf{X}\in \{0,1\}^d$ in this work. In real applications, the joint distribution of features and outcome can vary across environments: $P^{e}_{XY} \neq P^{e'}_{XY}$ for $e,e'\in\mathcal{E}$. The definition of stable prediction \citep{kuang2018stable} is then given as follows:
   \begin{problem}[Stable Prediction]
   Given one training environment $e\in \mathcal{E}$ with dataset $\mathbf{M}^{e}=\{\mathbf{X}^{e},Y^{e}\}$, the task is to \textbf{learn} a predictive model that can \textbf{stably} predict across unknown test environments $\mathcal{E}$.
   \end{problem}
   Here, we measure the performance of stable prediction by Average\_Error and Stability\_Error \citep{kuang2018stable},
   \begin{align*}
   \Scale[0.9]{\mbox{Average\_Error}} = \Scale[0.9]{ \frac{1}{|\mathcal{E}|}\sum\limits_{e \in \mathcal{E}}\mbox{Error}(\mathbf{M}^e)},\quad
   \Scale[0.9]{\mbox{Stability\_Error}}=
   \Scale[0.9]{\sqrt{\frac{1}{|\mathcal{E}|-1}\sum\limits_{e \in \mathcal{E}}\left(\mbox{Error}(\mathbf{M}^e)-\mbox{Average\_Error}\right)^{2}}},
   \end{align*}
   where $\mbox{Error}(\mathbf{M}^e)$ represents the predictive error on dataset $\mathbf{M}^e$.
   Now let $\mathbf{X} = \{\mathbf{S},\mathbf{V}\}$, where $\mathbf{S}$ denotes stable features and $\mathbf{V}$ denotes noisy features with following assumption \citep{kuang2018stable}:
   \begin{assumption}
   \label{asmp:stable}
   There exists a probability function $P(y|s)$ such that for all environment $e\in \mathcal{E}$, $P(Y^e=y|\mathbf{S}^e=s, \mathbf{V}^e = v) = P(Y^e=y|\mathbf{S}^e=s) = P(y|s)$.
   \end{assumption}
   Under assumption \ref{asmp:stable}, one can address the
   stable prediction problem by developing a predictive model that learns the stable function $f(\mathbf{S})$ induced by $P(y|s)$. For example, we have $f(\mathbf{S})=\mathbb{E}(Y|\mathbf{S})=\int yP(y|s)dy$ when $Y=f(\mathbf{S})+\varepsilon$ with the zero mean error $\varepsilon$. 
   In practice, we have no prior knowledge on which features belong to $\mathbf{S}$ and which belong to $\mathbf{V}$.
   
   In this paper, we study the stable prediction problem under model misspecification. For simplicity, we only discuss the regression case, and the classification scenario can be similarly derived. 
   Suppose that the true stable function $f(\mathbf{S})$ and $Y$ in environment $e$ are given by:
   \begin{eqnarray}
   \label{eq:Y_introduction}
   Y^e = f(\mathbf{S}^e)+\mathbf{V}^e\beta_V+\epsilon^e = \mathbf{S}^e\beta_S + g(\mathbf{S}^e) +\mathbf{V}^e\beta_V+ \varepsilon^e,
   \end{eqnarray}
   where $\beta_V = \mathbf{0}$ and $\varepsilon^e \perp \mathbf{X}^e$. We assume that the analyst mis-specifies the model by omitting non-linear term $g(\mathbf{S}^e)$ and uses a linear model for prediction. Then, standard linear regression may estimate non-zero effects of noisy features $\mathbf{V}^{e}$ if they are correlated with the omitted term $g(\mathbf{S}^e)$ in the training environment $e$, which leads to instability on prediction since the following theorem implies that the correlation between $\mathbf{V}$ and $g(\mathbf{S})$ is changeable across unknown test environments. 
   \begin{theorem}
   \label{theo:distribution_shift}
   Under assumption \ref{asmp:stable}, the distribution shift across environments is induced by the variation in the joint distribution over $(\mathbf{V},\mathbf{S})$.
   \end{theorem}
   
   \begin{proof}
    \begin{align}
    \nonumber P(\mathbf{X}^e, Y^e) &=  P(Y^e|\mathbf{X}^e)P(\mathbf{X}^e) = P(Y^e|\mathbf{S}^e,\mathbf{V}^e)P(\mathbf{S}^e,\mathbf{V}^e)\\
    \nonumber &= P(Y^e|\mathbf{S}^e)P(\mathbf{S}^e,\mathbf{V}^e) 
    \end{align}
    With assumption 1, we know the distribution $P(Y^e|\mathbf{S}^e) = P(Y^{e'}|\mathbf{S}^{e'})$ for different $e,e'\in\mathcal{E}$. Hence, the distribution shift across environments (\ie, $P(\mathbf{X}^e, Y^e) \neq P(\mathbf{X}^{e'}, Y^{e'})$) is induced by the variation in the joint distribution over $(\mathbf{V},\mathbf{S})$ (\ie, $P(\mathbf{S}^e,\mathbf{V}^e) \neq P(\mathbf{S}^{e'},\mathbf{V}^{e'}))$.
    \end{proof}
   
   \textbf{Notations.} Let $n$ refer to the sample size, and $d$ be the dimensionality of variables. 
   For any vector $\textbf{v} \in \mathbb{R}^{d\times 1}$, let $\|\textbf{v}\|_1 = \sum_{i=1}^{d}|v_i|$.
   For any matrix $\mathbf{X}\in \mathbb{R}^{n\times d}$, let $\mathbf{X}_{i,\cdot}$ and $\mathbf{X}_{\cdot,j}$ represent the $i^{th}$ sample and the $j^{th}$ variable in $\mathbf{X}$, respectively. To simplify notations, we remove the environment variable $e$ from $\mathbf{X}^e$, $\mathbf{S}^e$, $\mathbf{V}^e$, $\varepsilon^e$, and $Y^e$ when there is no confusion from the context.
   
    
    \section{Variables Deconfounding}
   \subsection{Generalized Global Balancing Loss}
   Theorem \ref{theo:distribution_shift} implies that if the covariates are mutually independent (or there are no confounding effects among variables), we can well estimate parameter $\beta_V$ in Eq. (\ref{eq:Y_introduction}), hence improve the stability of prediction across unknown test environments. The confounding effects between covariates and the binary treatment status are typically eliminated by balancing covariates in causality literature \citep{athey2018approximate,austin2011introduction,hainmueller2012entropy}.
   Recently, \cite{kuang2018stable} successively regarded each variable as the treatment indicator and minimized a global balancing loss:
   \begin{align}\label{eq:L_balancing}
     \Scale[1]{\min\limits_{\mathbf{W}\in\mathbb{R}^n}\mathcal{L}(\mathbf{W},\mathbf{X})} 
      &=\Scale[1]{\sum\limits_{j=1}^{d}\left\|\frac{\mathbf{X}_{\cdot,-j}^T\cdot (\mathbf{W}\odot \mathbf{X}_{\cdot,j})}{\mathbf{W}^T\cdot \mathbf{X}_{\cdot,j}}-\frac{\mathbf{X}_{\cdot,-j}^T\cdot (\mathbf{W}\odot (\mathbf{1}-\mathbf{X}_{\cdot,j}))}{\mathbf{W}^T\cdot (\mathbf{1}-\mathbf{X}_{\cdot,j})}\right\|_2^2} \nonumber
         \\
      &=\Scale[1]{
            \sum\limits_{j=1}^d\sum\limits_{k\neq j}\left[\frac{\sum_{i:X_{ij}=1}W_iX_{ik}}{\sum_{i:X_{ij}=1}W_i}-\frac{\sum_{i:X_{ij}=0}W_iX_{ik}}{\sum_{i:X_{ij}=0}W_i}\right]^2,}
   \end{align}
   where $\odot$ refers to Hadamard product;  $\mathbf{X}_{\cdot,-j} = \mathbf{X} \backslash \{\mathbf{X}_{\cdot,j}\}$ means all the remaining variables by removing the $j^{th}$ variable in $\mathbf{X}$; and $X_{ij}$ denotes the $(i,j)$ entry in $\mathbf{X}$. The difference in quadratic loss enforces $P_{\mathbf{W}}(X_k=1|X_j=1)\approx P_{\mathbf{W}}(X_k=1|X_j=0)$ w.r.t. the weighted conditional distribution $P_{\mathbf{W}}$. When the equation holds exactly, it can be shown that $X_k\in\{0,1\}$ and $X_j\in\{0,1\}$ are independent and thus have no confounding effects. $\mathcal{L}(\mathbf{W},\mathbf{X})$ hence globally balances each variable with others by reweighting the observations. 
   Despite effectiveness, Eq. (\ref{eq:L_balancing}) can only remove the first-order confounding effect among variables, while ignoring the higher-order ones, for example, between $\mathbf{V}$ and a $k$-way interaction function $g(\mathbf{S})$. Moreover, it is not a convex optimization problem and the global optima is hard to find. Finally, using full weighted data can be computationally expensive for further training especially in the big data scenarios.
   
   Considering high-order confounding effects among variables, we define a new \emph{generalized global balancing} loss $\mathcal{L}_{k}(\mathbf{W}, \mathbf{X})$ as:
   \begin{align}\label{eq:GL_balancing}
   \Scale[1]{
   \mathcal{L}_{k}(\mathbf{W}, \mathbf{X})} &=
   \Scale[1]{\sum\limits_{j\in[d]}\sum\limits_{I_k\subseteq[d]\backslash\{j\}}\left[\frac{\mathbf{X}_{I_k}^T\cdot (\mathbf{W}\odot \mathbf{X}_{\cdot,j})}{\mathbf{W}^T\cdot \mathbf{X}_{\cdot,j}}-\frac{\mathbf{X}_{I_k}^T\cdot (\mathbf{W}\odot (\mathbf{1}-\mathbf{X}_{\cdot,j}))}{\mathbf{W}^T\cdot (\mathbf{1}-\mathbf{X}_{\cdot,j})}\right]^2}\nonumber\\
   & = \Scale[1]{ \sum\limits_{j\in[d]}\sum\limits_{I_k\subseteq[d]\backslash\{j\}}\left[\frac{\sum_{i:x_{ij}=1}W_iX_{iI_k}}{\sum_{i:X_{ij}=1}W_i}-\frac{\sum_{i:X_{ij}=0}W_iX_{iI_k}}{\sum_{i:X_{ij}=0}W_i}\right]^2,
   }
   \end{align}
   where $k$ refers to the order of confounding effect with $1\leq k<d$, and $X_{iI_k}, \mathbf{X}_{I_k}$ denote the $k$-way interaction w.r.t. the index subset $I_k$. This loss broadly measures different orders of correlation or confounding effect between $\mathbf{V}$ and $g(\mathbf{S})$. It is easy to see that Eq. (\ref{eq:L_balancing}) is a special case of Eq. (\ref{eq:GL_balancing}) with $k=1$. Our target in this paper is to minimize the aggregation of $L_k(\mathbf{W},\mathbf{X})$ up to the order $k$.
    
    \subsection{Variables Deconfounding via FFDs}
    In this section, we elaborate on how FFDs can be used to deconfound the variables in terms of minimizing the generalized balancing loss in Eq. (\ref{eq:GL_balancing}). Note that the binary-encoded data matrix is closely related to a two-level factorial design, which motivates us to leverage the classical results from the fractional factorial design literature. 
    
    \textbf{Two-level fractional factorial design (FFD) \citep{dey2009fractional}:} It is a size-$m$ subset of the full factorial design that consists of all $2^{d}$ possible combinations of the vector $\{-1,1\}^d$.
    We denote FFD by $\D\in\{-1,1\}^{m\times d}$, where $0<m\leq 2^d$.
    
    One important feature of FFD is that variables and their interactions are orthogonal to some degrees, and they can achieve joint orthogonality when FFD becomes full factorial. Another cardinal observation is that the mean differences in Eq. (\ref{eq:GL_balancing}) can be transferred into the inner products of the main effects and high-order interactions of a design in $\{-1,1\}^{m\times d}$; see the proof of Theorem \ref{theo:loss_optimal}. Consequently, the orthogonality of FFD can help remove non-zero inner products and lead to a minimal loss.
    
    Resolution \citep{fries1980minimum}, denoted as $R$, is an important criterion to reflect the degree of orthogonality. For an FFD, define the \textit{generalized word-length pattern} \citep{ma2001note} 
    \begin{align*}
    \Scale[1]{W(\D)=(A_1(\D),\dots,A_d(\D)),}
    \end{align*}
    where $A_j(\D)$ refers to the generalized wordlength and
    measures the degree of $j$-factor non-orthogonality.
    Specifically,
    \begin{align}
    \label{eq:GWL}
    \Scale[1]{
    A_j(\D) = \frac{1}{m(q-1)}\sum_{k=0}^dP_j(k;d,q)B_j(\D),\quad j=1,\dots,d,}
    \end{align}
    where $q$ denotes the number of levels,
    $$
    \Scale[1]{
    P_j(x;d,q)=\sum_{w=0}^j(-1)^w(q-1)^{j-w}\tbinom{x}{w}\tbinom{d-x}{j-w}}
    $$
    are the Krawtchouk polynomials \citep{macwilliams1977theory}, and $B(\D)=(B_0(\D),\dots,B_d(\D))$ is the distance distribution 
    $B_j(\D)=m^{-1}|\{(\cc,\dd):d_{H}(\cc,\dd)=j,\cc,\dd\in\D\}|$ with $d_{H}(\cdot,\cdot)$ denoting the Hamming distance. Note that $B_j(\D)$ is invariant to the encoding way of $\D$. 
    Then, the resolution of $\D$ is defined as the smallest index $R\leq d$ such that $A_{R}(\D)> 0$. Note that the full factorial design with $m=2^d$ has resolution $d+1$, since $A_j(\D)=0$ for all $j\in[d]$.
    
    The following lemma \citep{hedayat2012orthogonal} explains the relationship between resolution $R$ and orthogonal strength. We omit `fractional factorial' in the resolution-$R$ design without ambiguity to the context.
    \begin{lemma}\label{lemma-orth}
    The resolution-$R$ design $\mathbf{D}$ has orthogonal strength $t=R-1$, where $t$ means that one can see all possible $t$-tuples equally often for $m/2^t$ times in any $t$ columns of $\mathbf{D}$.
    \end{lemma}
    
    Lemma \ref{lemma-orth} implies that any $t$ columns/variables of $\mathbf{D}$ contain $m/2^t$ full factorials such that every $t$ factors are jointly orthogonal. Furthermore, this lemma also implies low order $t'<t$ orthogonal strength exists for the resolution-$R$ design. In other words, FFD can preserve the joint orthogonality up to its resolution minus one. For example, a resolution-3 design guarantees the pairwise orthogonality among the main effects of all factors.
    The following theorem further states the preserved orthogonality among the main effect and their $k$-way interaction. 
    
    \begin{theorem}\label{theo:orthogonal_with_k_interactions}
    Let $I_k\subseteq[d]$ denote any collection of distinctive factors with $|I_k|=k\leq d$, $\D=(\D_{\cdot1},\dots,\D_{\cdot d})\in\{-1,1\}^{m\times d}$ be the design matrix, and $\D_{I_k}\in\{-1,1\}^{m}$ represent $k$-way interaction of $I_k$.
    We have a) $\D_{\cdot i}^T\D_{\cdot j}=0, i\neq j,$ for any resolution-$R$ design with $R\geq 3$; and b)
    $\D_{I_k}^T\D_{\cdot j}=0, j\in[d], 2\leq k\leq R-2,$ for any resolution-$R$ design with $R\geq4$.
    \end{theorem}
    \begin{proof}
     To prove above theorem, we first inductively show a lemma that any full factorial design (FD) denoted by $\D\in\{-1,1\}^{2^d\times d}$ has $\mathbf{1}^T\D_{I_d}=0$ for the integer $d\geq 1$. It is easy to check that $\mathbf{1}^T\D_{I_d}=0$ holds when $d=1,2$. Suppose this equality holds for any integer $d=\ell, \ell\geq1$. When $d=\ell+1$, note that $\D$ is invariant to the permutation of rows, so we rearrange the first column and have
     $$\D_{I_{\ell+1}}=(\underbrace{-1,\dots,-1}_{2^\ell},\underbrace{1,\dots,1}_{2^\ell})^T\odot(\D^{(1)}_{I_{\ell}},\D^{(2)}_{I_{\ell}})^T,$$
     where the sub-designs $\D^{(1)},\D^{(2)}\in\{-1,1\}^{2^\ell\times \ell}$ also belong to FD \citep{hedayat2012orthogonal}.
     Therefore, it can be derived that $\mathbf{1}^T\D_{I_{\ell+1}}=\mathbf{1}^T\D^{(2)}_{I_\ell}-\mathbf{1}^T\D^{(1)}_{I_\ell}=0$ and the statement gets proved. So for any FFD with orthogonal strength $t$, we have $\mathbf{1}^T\D_{I_k}=0$ for $k\in[t]$, because Lemma 1 tells that all combinations of at most $t$-tuples (full factorial design) appear with equal frequency in corresponding distinctive columns.
     
     With the above property, we can easily show the first case in the theorem as $\D_{\cdot i}^T\D_{\cdot j}=\mathbf{1}^T\D_{I_2}$ with $I_2=\{i,j\}$, and the resolution-$R$ design has orthogonal strength $t=R-1\geq 2$, which implies $\D_{\cdot i}^T\D_{\cdot j}=0$. For the second case, we restate it in terms of orthogonal strength $t$, that is, we need to show $\D^T_{I_k}\D_{\cdot j}=0,$ $j\in[d], 2\leq k\leq t-1$ for $t\geq 3$, which can be inductively proved in the similar manner. Without loss of the generality, we just show $\D^T_{I_k}\D_{\cdot j}=0$ for $t=3$ in what follows. When $j\notin I_2$, we can construct a 3-column FFD with indices $I_3=\{j\}\cup I_2$ and $\D^T_{I_2}\D_{\cdot j}=\mathbf{1}^T\D_{I_{3}}$. And we can similarly obtain $\mathbf{1}^T\D_{I_{3}}=0$ as done in the first case because of $t=3$. If $j\in I_2=\{i,j\}$, it is easy to check that $\D^T_{I_2}\D_{\cdot j}=\mathbf{1}^T\D_{\cdot i}=0$ as $t=3$. For $t\geq 4$, 
     since high-order orthogonal strength implies the low order ones, we only need to consider the situation of $k=t-1$. And we can similarly obtain the conclusion by discussing $j$ in $I_k$ or not.
     \end{proof}
    
    With these results of FFD, if we determine the subdata matrix $\mathbf{X}\in\{0,1\}^{m\times d}$ by exactly matching it to some resolution-$R$ design $\D\in\{-1,1\}^{m\times d}$ with the rule $\mathbf{D}_{\cdot j}=2\mathbf{X}_{\cdot j}-\mathbf{1}$.
    We can show that such $\mathbf{X}$ is the optimal solution of Eq. (\ref{eq:GL_balancing}) with weights $\mathbf{W}=\mathbf{1}$.
    \begin{theorem}\label{theo:loss_optimal}
    For any $\mathbf{X}\in\{0,1\}^{m\times d}$ matching the resolution-$R$ design with $R\geq 3$,
    we have
    $\mathcal{L}_{k}(\mathbf{W},\mathbf{X})=0$ for any $1\leq k\leq R-2$ and $\mathbf{W}=\mathbf{1}$.
    \end{theorem}
    \begin{proof}
     Let $\D\in\{-1,1\}^{m\times d}$ be the resolution-$R$ design matched with $\mathbf{X}$. For any $I_k\subseteq[d]\backslash\{j\}$ with a given $j$ and feasible $k$, let $I_k=\{j_1,\dots,j_k\}$ and we have
     \begin{align*}
     \mathbf{X}_{I_k}
     &= \frac{1}{2^k}(\D_{\cdot j_1}+\mathbf{1})\odot\dots \odot (\D_{\cdot j_k}+\mathbf{1})\\
     &= \Scale[1.0]{\frac{1}{2^k}\left(\mathbf{1}+\sum_{h=1}^{k}\sum_{\tilde{I}_h\subseteq I_k} \D_{\tilde{I}_h}\right),}
     \end{align*}
     where $\tilde{I}_h$ is the subset of $I_k$ with cardinality $h$.
     When $W=\mathbf{1}$, we have
     \begin{align*}
     &\Scale[1.0]{\left[\frac{\mathbf{X}_{I_k}^T\cdot (\mathbf{W}\odot \mathbf{X}_{\cdot,j})}{\mathbf{W}^T\cdot \mathbf{X}_{\cdot,j}}-\frac{\mathbf{X}_{I_k}^T\cdot (\mathbf{W}\odot (\mathbf{1}-\mathbf{X}_{\cdot,j}))}{\mathbf{W}^T\cdot (\mathbf{1}-\mathbf{X}_{\cdot,j})}\right]^2}\\
     =&\Scale[1.0]{\left[\frac{2}{m}\mathbf{X}_{I_k}^T\left(2\mathbf{X}_{\cdot j}-\mathbf{1}\right)\right]^2=\left(\frac{2}{m}\mathbf{X}_{I_k}^T\mathbf{D}_{\cdot j}\right)^2} \\
     =&\Scale[1.0]{\frac{1}{4^{k-1}m^2}\left(\mathbf{1}^T\D_{\cdot j}+\sum_{h=1}^{k}\sum_{\tilde{I}_h\subseteq I_k} \D_{\tilde{I}_h}^T\D_{\cdot j}\right)^2=0,}
     \end{align*}
     where the last equality follows Theorem 2. Specifically, when $k=1$ or $R=3$, we have $\tilde{I}_h=I_k=\{j_k\}$ with $j_k\neq j$ and the last equality becomes zero according to the case a) in Theorem 2. Similar results can be derived for $2 \leq k\leq R-2$ ($R\geq 4$) following the case b).
     Consequently, it is evident that $\mathcal{L}_{k}(\mathbf{1},\mathbf{X})$ equals to zero for any $1\leq k\leq R-2$ ($R\geq 3$).
     \end{proof}
    
    This theorem reveals that higher resolution design can lead to more stable outcomes, as the lower-order confounding effects are removed by the perfect balance. Since a higher resolution design would require a larger run size $m$. In the present paper, we use resolution-5 design as a subsampling template,  which ensures $\mathcal{L}_k(\mathbf{1},\mathbf{X})=0$ for $k=1,2,3$. The template can be easily generated from open source packages, such as \texttt{FrF2} in R \citep{gronmping2014r}. Finally, we can significantly save the calculation time for model training based on selected data, as $m$ is typically much smaller than full data size $n$.
    
   We end this section with a toy example in Fig. \ref{fig:toy_example} which illustrates the main idea of different deconfounding methods.
   Consider a three-dimensional data set with binary input. We visualize the sample space in Fig. \ref{fig:raw} and the bubble size corresponds to the number of observations on that point. Note that each facet corresponds to the conditional distribution of two variables w.r.t. the remaining one. Therefore, all bubbles should have the same size in the ideal case when there are no confounding effects among variables, since any two opposite facets should have the same distribution.  To achieve this goal, global balancing method (Fig. \ref{fig:reweigh}) takes all sample values but reweights them to change the data distribution. It is easy to see that the ideal case accords with the full factorial design where all possible sample values appear with the same frequencies. So our subsampling method (Fig. \ref{fig:subsample}) only uses a fraction of samples that well represent the ideal situation. For example,
   the conditional distributions of $X_i|X_k$ and $X_j|X_k$ are the same in Fig. \ref{fig:subsample}, where $i,j,k\in\{1,2,3\}$ are distinctive indices. 
   
   \begin{figure*}[htb]
    \centering
    \vspace{-0.1in}
    \subfloat[Raw data \label{fig:raw}]{
      \includegraphics[width=1.66in]{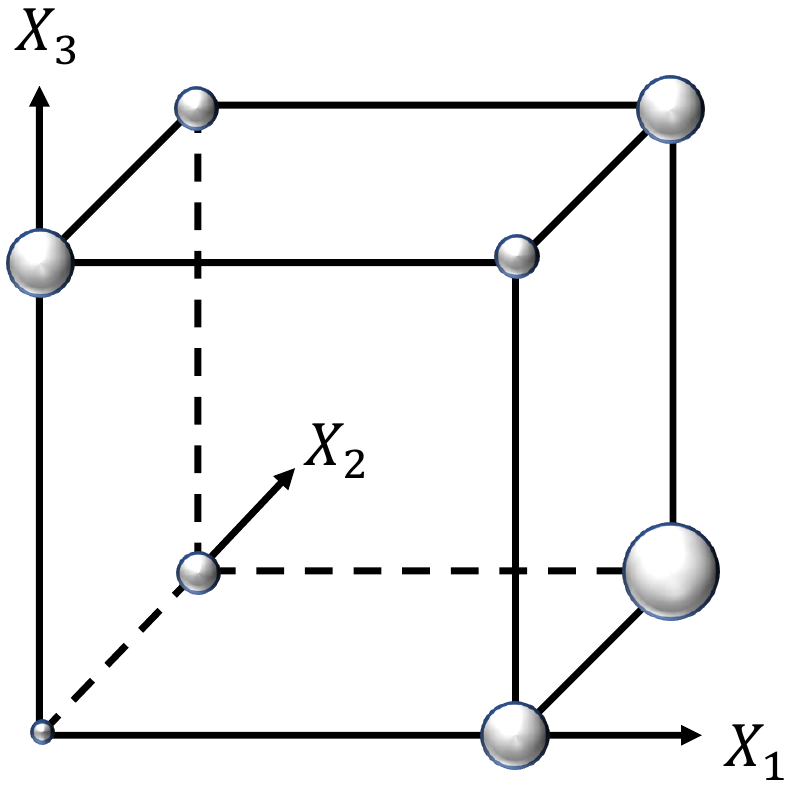}
    }
    \ \ 
    \subfloat[Reweighting\label{fig:reweigh}]{
      \includegraphics[width=1.66in]{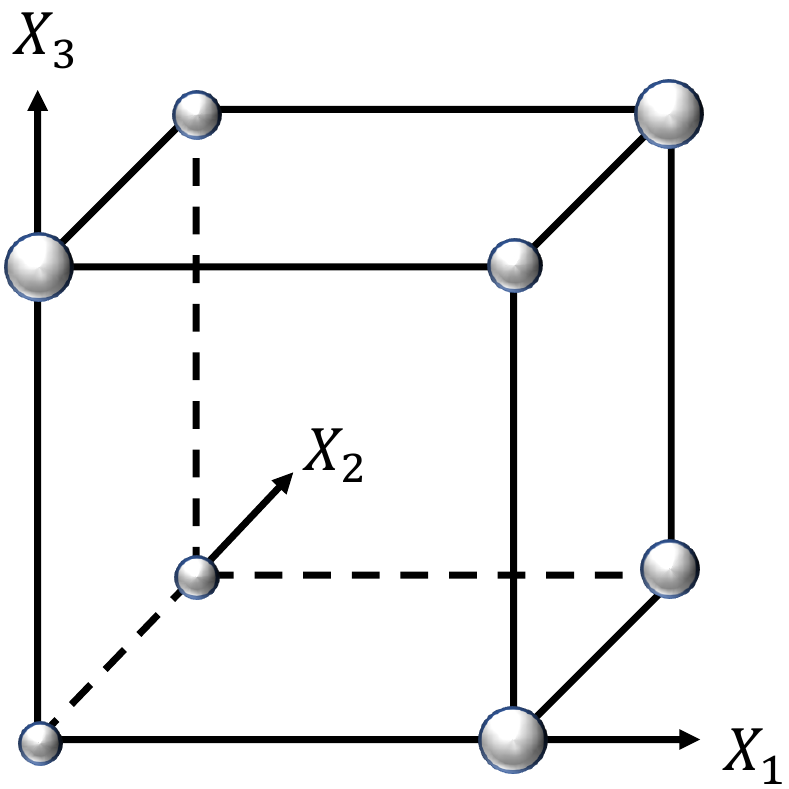}
    }
    \ \ 
    \subfloat[Subsampling\label{fig:subsample}]{
      \includegraphics[width=1.66in]{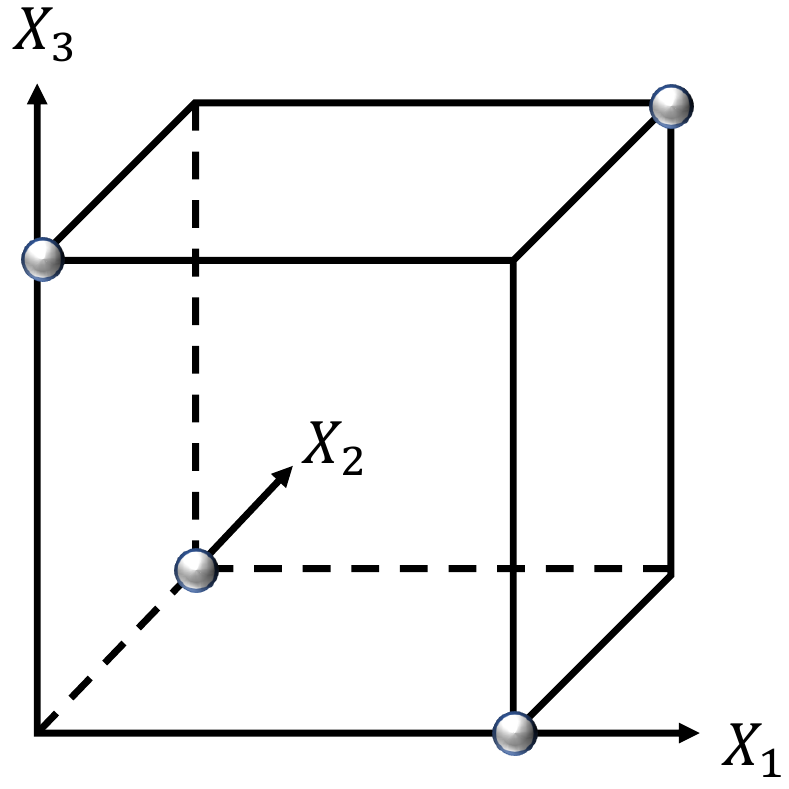}
    }
    \vspace{-0.1in}
    \caption{A toy example to illustrate the main idea of each deconfounding method.}
    \label{fig:toy_example}
    \vspace{-0.1in}
    \end{figure*}

    \section{BSSP Algorithm}
    BSSP algorithm consists of an FFD-based subsampling method and a subsampled learning model.  We first introduce the specific subsampling algorithm that is feasible for general situations. 
    To obtain a balanced subdata with deconfounded variables, we propose a matching algorithm based on the FFD template.  Given a resolution-$R$ design $\D\in\{-1,1\}^{m\times d}$, we transfer its encoding into $\{0,1\}$. Then we select the samples from $\mathbf{M} = \{\mathbf{X}\in \{0,1\}^{n\times d}, Y\in \mathbb{R}^{n}\}$ if some row in $\D$ can match the one in $\mathbf{X}$. 
    The matching processes are described in Algorithm \ref{alg:matching}.
    
    \begin{algorithm}[htb]
    \caption{{Sample\_Matching Algorithm}}
    \label{alg:matching}
    \begin{algorithmic}[1]
    \Require
    Observed samples $\mathbf{M} = \{\mathbf{X}\in \{0,1\}^{n\times d}, Y\in \mathbb{R}^{n}\}$ and a design $\D\in \{0,1\}^{m\times d}$
    \Ensure
    A subset of samples $\mathbf{M}_{\rm sub}$
    \State Set $\mathbf{M}_{\rm sub} = \emptyset$
    \For{each row/sample $\D_{i,\cdot} \in \D$ }
    \If{$\D_{i,\cdot} == \X_{j,\cdot}$}
    \State $\mathbf{M}_{\rm sub} = \mathbf{M}_{\rm sub} \cup (\X_{j,\cdot}, Y_{j})$
    \State break
    \EndIf
    \EndFor\\
    \Return $\mathbf{M}_{\rm sub}$
    \end{algorithmic}
    \end{algorithm}
    
    However, it may not be easy to achieve a perfect matching and thus non-confounding properties in practice. Note that
    Lemma \ref{lemma-orth} implies the orthogonality is invariant to the column permutation of $\mathbf{D}$,  which we may denote as $\mathscr{D}$ with the cardinality $d!$. All the designs in $\mathscr{D}$ share the same orthogonal properties as the template design $\D$.
    Hence, we may find a better design $\D'$ in $\mathscr{D}$ such that all its design points can be matched to the observed samples.
    
    \begin{algorithm}[htb]
    \caption{{FFD-Based Subsampling Algorithm}}
    \label{alg:subsampling}
    \begin{algorithmic}[1]
    \Require
    Observed samples $\mathbf{M} = \{\mathbf{X}\in \{0,1\}^{n\times d}, Y\in \mathbb{R}^{n}\}$ and a resolution-$R$ design $\mathbf{D}\in\{0,1\}^{m\times d}$.
    \Ensure
    A subset of samples $\mathbf{M}_{\rm sub}$
    \State Set $\mathbf{M}_{\rm sub} = \emptyset$
    \State Generate a design set $\mathscr{D}$ by column permutation on $\D$,
    \For{$\D' \in \mathscr{D}$ }
    \State $\mathbf{M'}_{\rm sub} = $ $Sample\_Matching$ ($\mathbf{M}, \mathbf{D'}$)
    \State Calculate its confounding measure $\psi(\mathbf{M'}_{\rm sub})$
    \If{$\psi(\mathbf{M'}_{\rm sub})<\psi(\mathbf{M}_{\rm sub})$}
    \State Let $\mathbf{M}_{\rm sub} = \mathbf{M'}_{\rm sub}$
    \EndIf
    \If{$\psi(\mathbf{M}_{\rm sub})==0$}\Comment{ all samples in $\D'$ are matched}
    \State break
    \EndIf
    \EndFor\\
    \Return $\mathbf{M}_{\rm sub}$
    \end{algorithmic}
    \end{algorithm}
    
    If none of the designs in $\mathscr{D}$ can be fully matched to the observed samples in $\mathbf{X}$, 
    we propose a \textbf{confounding measure} to evaluate the degree of confounding properties for $\mathbf{M}_{\rm sub}=\{\widetilde{\mathbf{X}},\widetilde{Y}\}$: 
    \begin{eqnarray}
    \label{eq:confounding-measure}
    \Scale[1]{\psi(\mathbf{M}_{\rm sub})=\sum_{j=1}^{R-1}\rho^jA_j(\widetilde{\mathbf{X}}),\quad 0<\rho<1,}
    \end{eqnarray}
    where $A_j(\widetilde{\mathbf{X}})$ refers to the generalized wordlength in Eq. (\ref{eq:GWL}) and $\rho$ is a parameter for exponentially weighing. In the ideal case, we have $\psi(\mathbf{M}_{\rm sub})=0$ according to the definition of resolution-$R$ design. Note that $A_j(\widetilde{\mathbf{X}})$ reflects the severity of order-$j$ confounding effects and is invariant to the design encoding. Motivated by the famous \textit{effect hierarchy principle} \cite{wu2011experiments}: (i) lower-order effects are more likely to be important; and (ii) effects with the same order are equally likely to be important, we set $\rho=0.9$ to assign more weights to the lower-order effects.
     
     A simulation is performed shown by Fig. \ref{fig-measure} to validate that $\psi(\mathbf{M}_{\rm sub})$ can measure the deviation of $\tilde{\mathbf{X}}$ to the FFD, where we calculate the $\psi(\mathbf{M}_{\rm sub})$ on random subsets of a 128-run resolution-5 design with different sizes and 100 replications. As we can see, the measure goes to 0 without any variation when $\tilde{\mathbf{X}}$ is close to the template design.
          
          \begin{figure}[h]
          \centering
          \includegraphics[width=0.75\linewidth]{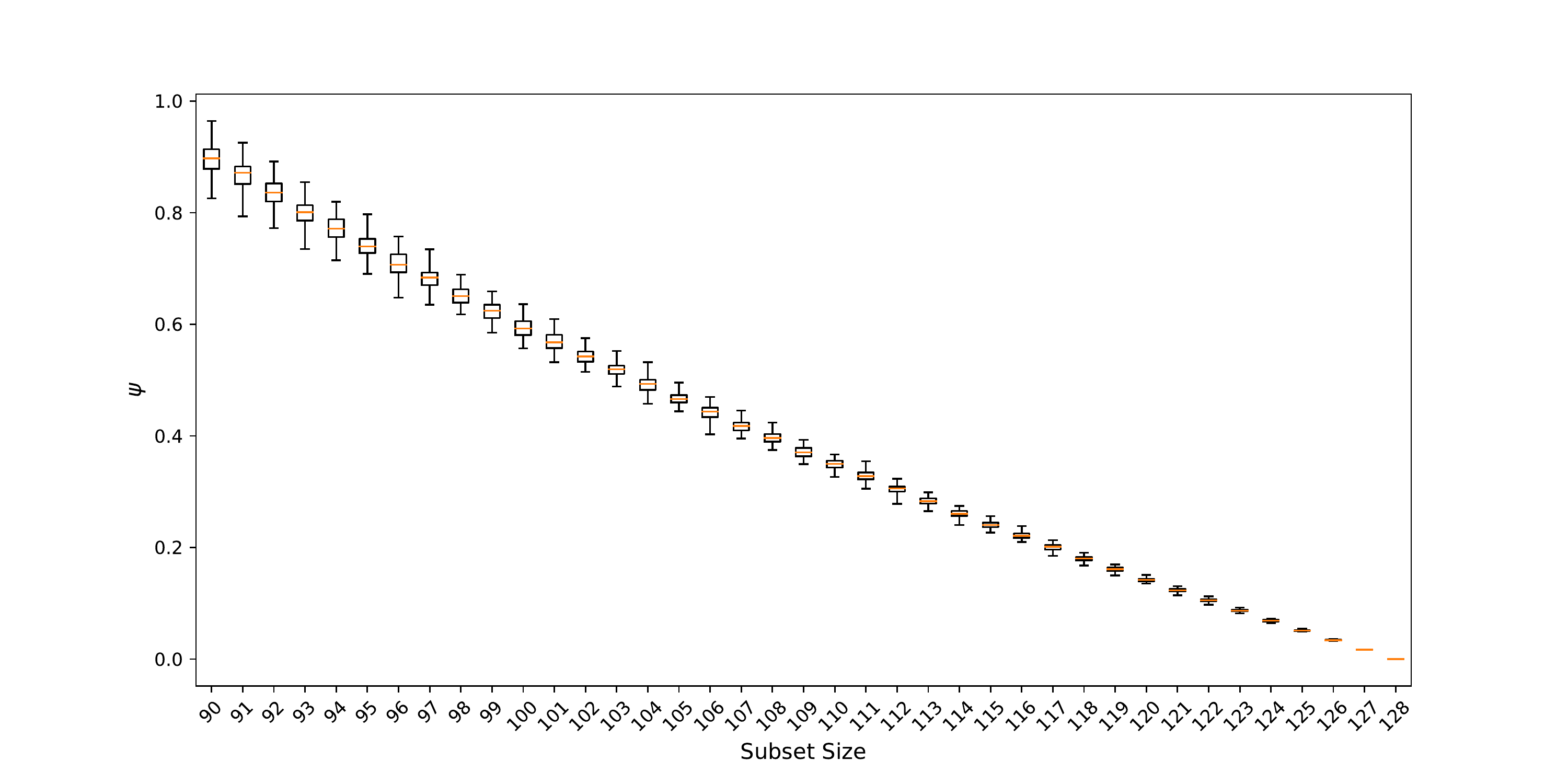}
          \caption{$\psi(\mathbf{M}_{\rm sub})$ on random subsets of a 128-run resolution-5 FFD with different subdata sizes. }
          \label{fig-measure}
          \end{figure}

    Based on Eq. (\ref{eq:confounding-measure}), one can calculate the confounding measure for each $\D$ in $\mathscr{D}$. Then, we can rank these subdata candidates and select the one with the minimal $\psi$-value. The details of the complete subsampling method are described in Algorithm \ref{alg:subsampling}.
    
    With the balance-subsampled data $\mathbf{M}_{\rm sub}$ from Algorithm \ref{alg:subsampling}, one can directly run a machine learning model for prediction, including regression for continuous outcome $Y$ and classification for categorized outcome $Y$. In this work, we simply consider the typical linear regression and logistic model with the original linear features of $\tilde{\X}$.  

    \section{Experiments}
    We compare our BSSP with three algorithms in the experiments. Traditional Logistic Regression (LR) and Ordinary Least Square (OLS) are set to be the baseline methods for classification and regression tasks, respectively. As for the benchmark of stable learning, we consider the Global Balancing Regression (GBR) \cite{kuang2018stable}. It learns the global weights for all samples in order to make the variables approximately non-confounding, and then performs weighted classification and regression by LR and OLS. Additionally, the LASSO regularizer \cite{tibshirani1996regression} is configured in all methods.  The performance of different approaches are then evaluated by RMSE of predicted outcomes, $\beta\mbox{\_Error}$  ($\|\beta-\hat{\beta}\|_1$), Average\_Error, and Stability\_Error, where we define
     $\mbox{Error}(\mathbf{M}^e)=\mbox{RMSE}(\mathbf{M}^e)$ in Average\_Error and Stability\_Error. A resolution-5 design with $m=128$ is used as the template for experiments.
    
    \subsection{Synthetic Datasets}
    \subsubsection{Stable Prediction for Regression Task}
    \paragraph{Datasets.}
    We generate binary predictors $\mathbf{X} =\{\mathbf{S}_{\cdot,1}, \cdots, \mathbf{S}_{\cdot,5}, \mathbf{V}_{\cdot,1}, \cdots, \mathbf{V}_{\cdot,5}\}$ with independent entries from $\mathcal{N}(0,1)$.
    Then we binarize the variable by setting $\mathbf{X}_{\cdot,j}=1$ if $\mathbf{X}_{\cdot,j}\geq 0$, otherwise $\mathbf{X}_{\cdot,j}=0$.
    Finally, we generate continuous response variable $Y$ following Eq. (\ref{eq:Y_introduction}), where $\beta_S = [\frac{1}{3},-\frac{2}{3},1,-\frac{1}{3},\frac{2}{3}]$, $\beta_V = \mathbf{0}$, and $\varepsilon \sim \mathcal{N}(0,0.3^2)$. The non-linear term $g(\cdot)$ is set to be $\mathbf{S}_{\cdot,1}\mathbf{S}_{\cdot,2}$. 
    To test the stability of the algorithm, we vary the environment via a biased sample selection with a rate of $r\in[-3,-1)\cup(1,3]$. Specifically, we select the sample with a probability $p=|r|^{-5\tau}$, where $\tau=|\mathbf{S}\beta_S+\mathbf{S}_{\cdot,1}\mathbf{S}_{\cdot,2}-\mbox{sign}(r)\mathbf{V}_5|$. The sign of $r$ determines the type of correlation between $Y$ and $\mathbf{V}_5$. $r>0$ refers to a positive correlation and $|r|$ quantifies the magnitude of correlation. We generate $n=2000$ samples after the biased selection.
    
    \vspace{-0.35cm}
    \paragraph{Results.}
    We train all models in the same environment with $r_{\rm train}=2$, and replicate the training for 50 times. We report mean and variance of $\beta\mbox{\_Error}$ under the training environment in Fig. \ref{fig:s0v_beta_S} \& \ref{fig:s0v_beta_V}. 
    To evaluate the stability of prediction, we test all models on various test environments with different bias rates $r_{\rm test}\in\{-3,-2,-1.7,-1.5,-1.3,1.3,1.5,1.7,2,3\}$.
    For each $r_{\rm test}$, we generate 50 independent test datasets and report the averaged RMSE in Fig. \ref{fig:s0v_rmse}. Given RMSEs, we report Average$\_$Error and Stability$\_$Error to evaluate the stability of prediction (See Fig. \ref{fig:s0v_stability}). Compared with baselines, the BSSP algorithm achieves a more precise estimation of the parameters and the best stable prediction across unknown  test environments with fewer samples. 
    
    \begin{figure*}[t]
    \centering
    \vspace{-0.2in}
    \subfloat[$\beta$\_Error of $\mathbf{S}$\label{fig:s0v_beta_S}]{
      \includegraphics[width=1.5in]{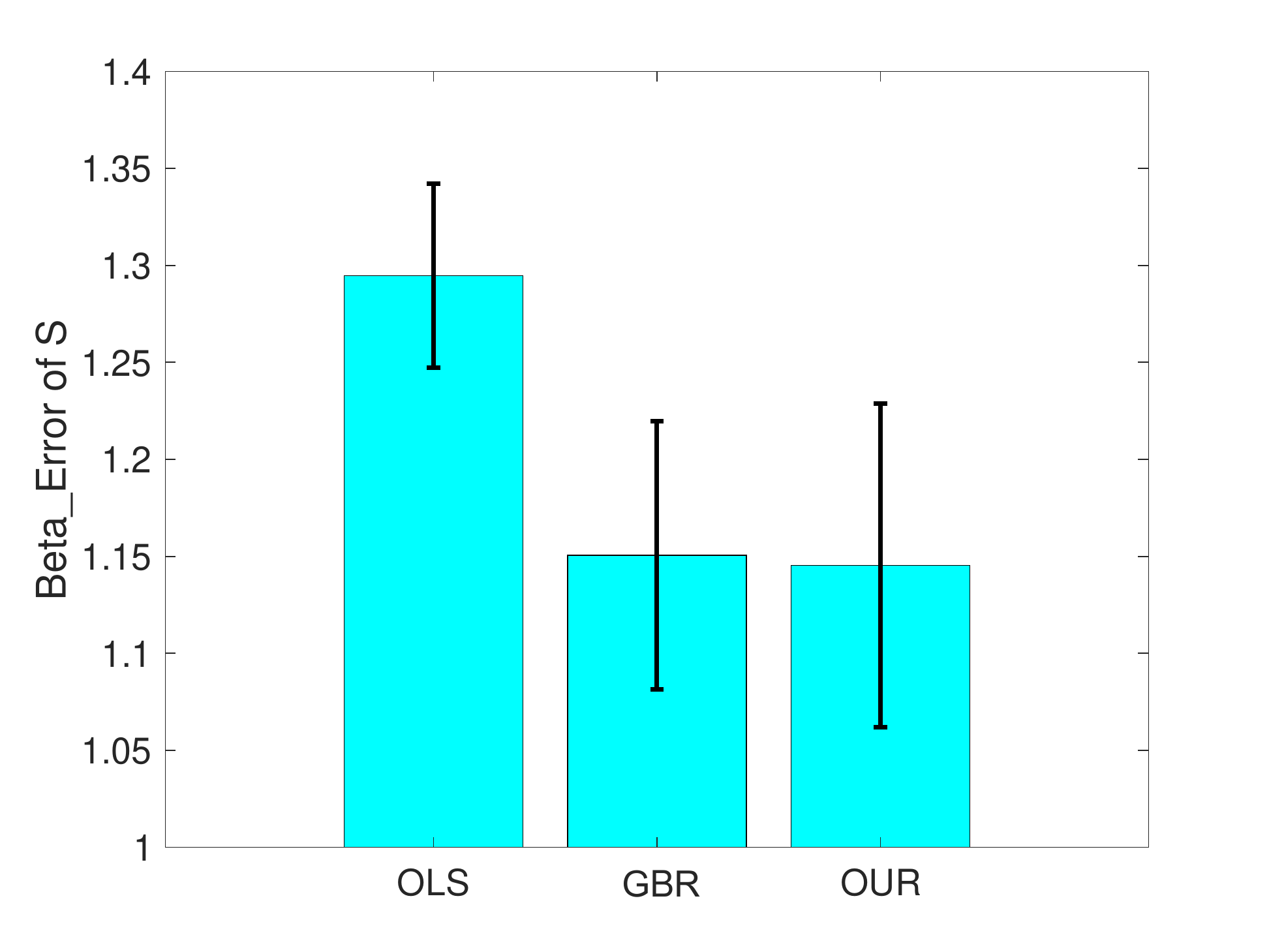}
    }
    \subfloat[$\beta$\_Error of $\mathbf{V}$\label{fig:s0v_beta_V}]{
      \includegraphics[width=1.5in]{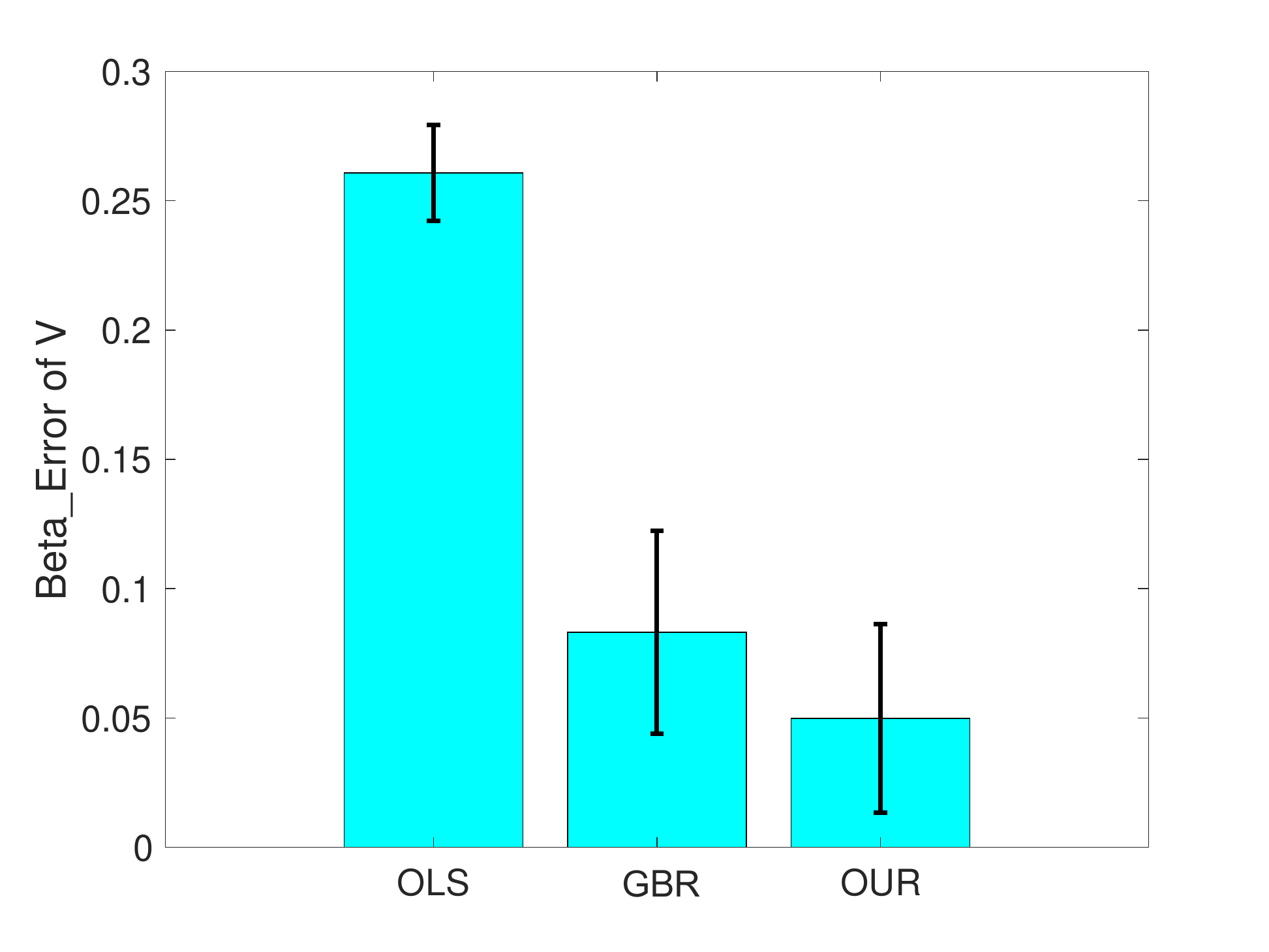}
    }
    \subfloat[RMSE over $r$ \label{fig:s0v_rmse}]{
      \includegraphics[width=1.5in]{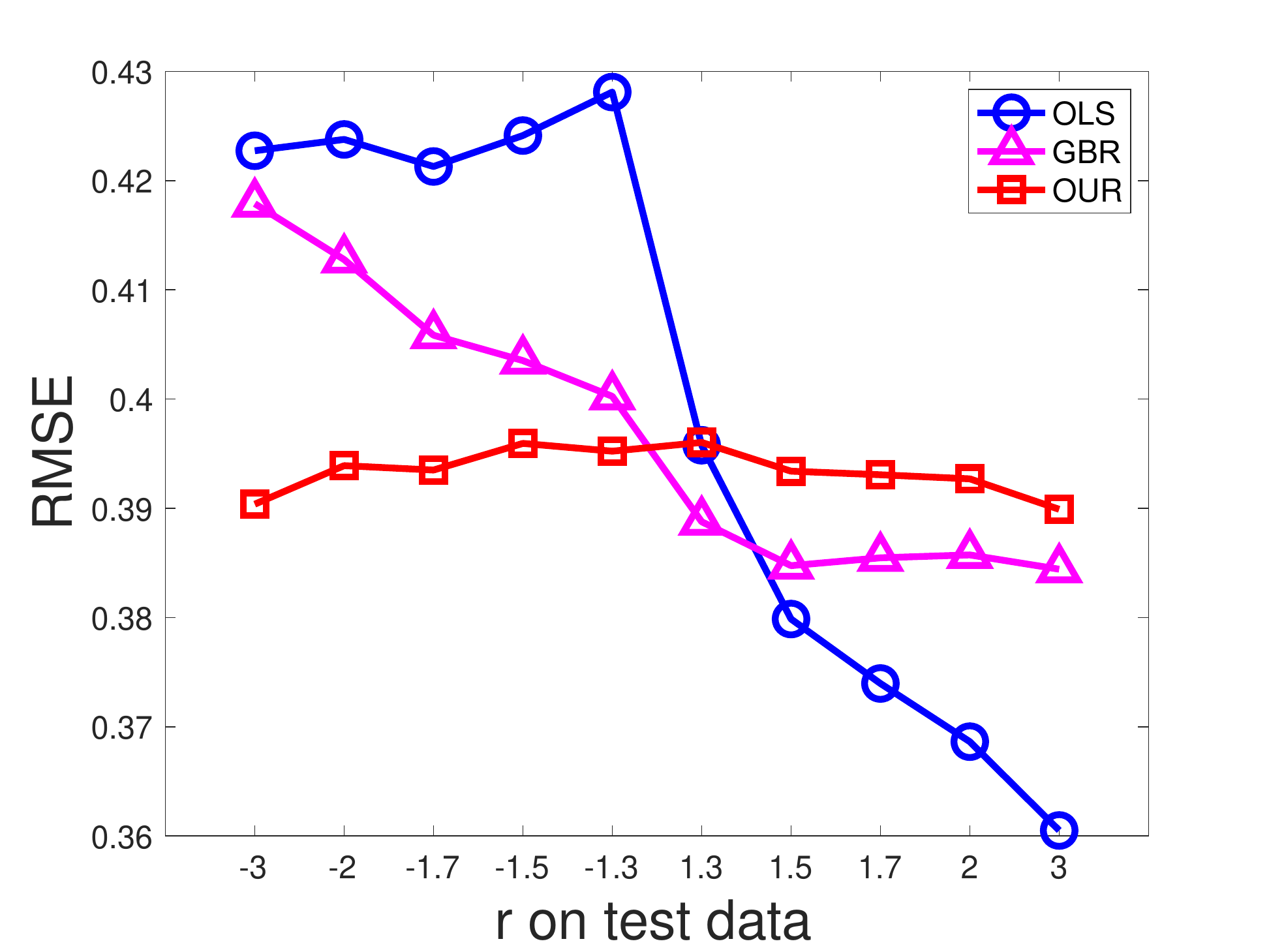}
    }
    \subfloat[Average\_Error (bar) $\&$ Stability\_Error (line) \label{fig:s0v_stability}]{
      \includegraphics[width=1.5in]{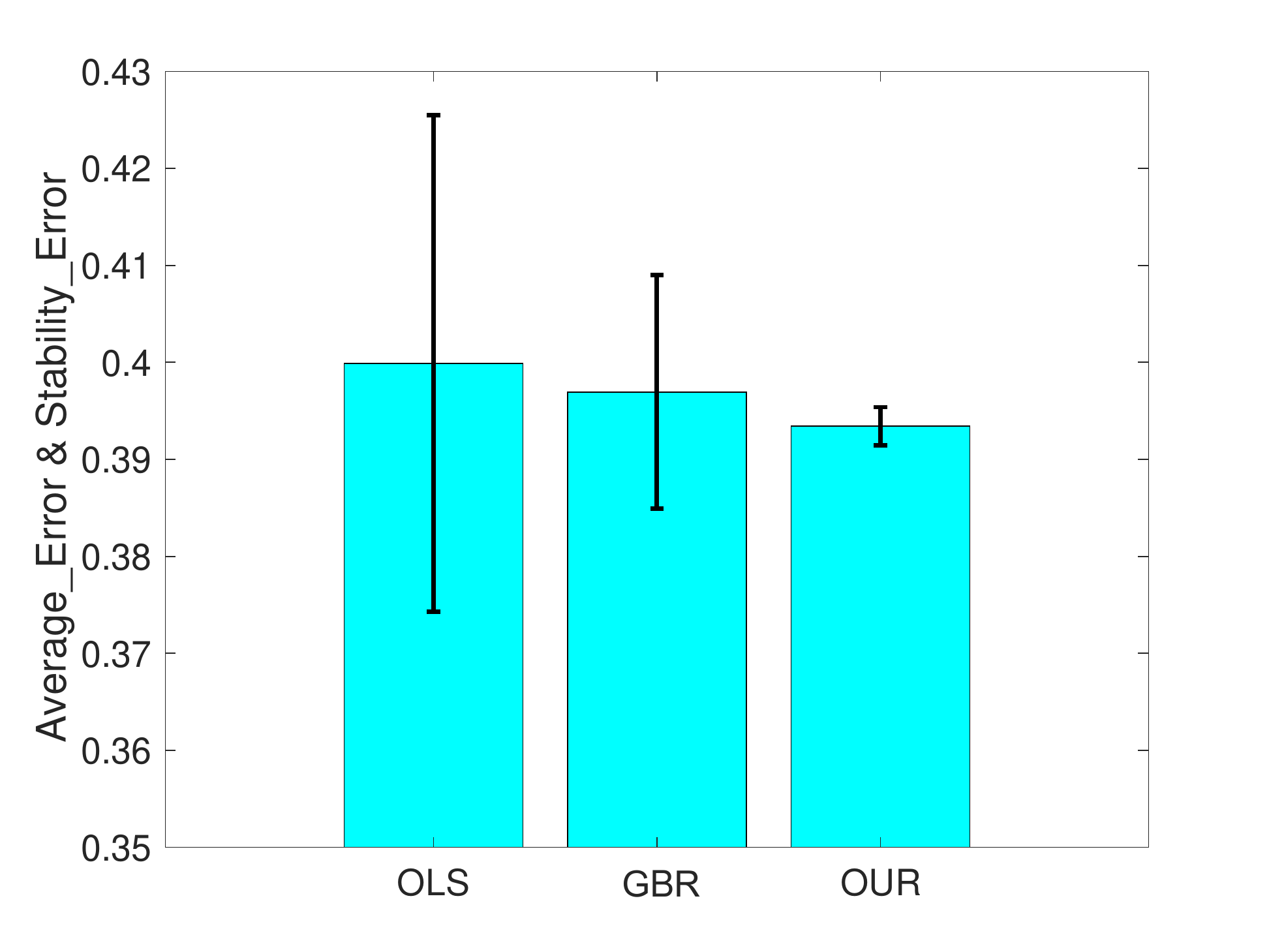}
    }
    \vspace{-0.05in}
    \caption{Results of regression. All the models are trained with $n=2000$, $d=10$ and $r_{\rm train} = 2.0$.}
    \label{fig:s0v}
    \vspace{-0.1in}
    \end{figure*}
    \begin{figure*}[tb]
    \centering
    \vspace{-0.2in}
    \subfloat[\scriptsize{$r_{\rm train}=0.15$} \label{fig:C_RMSE_128-015}]{
      \includegraphics[width=1.5in]{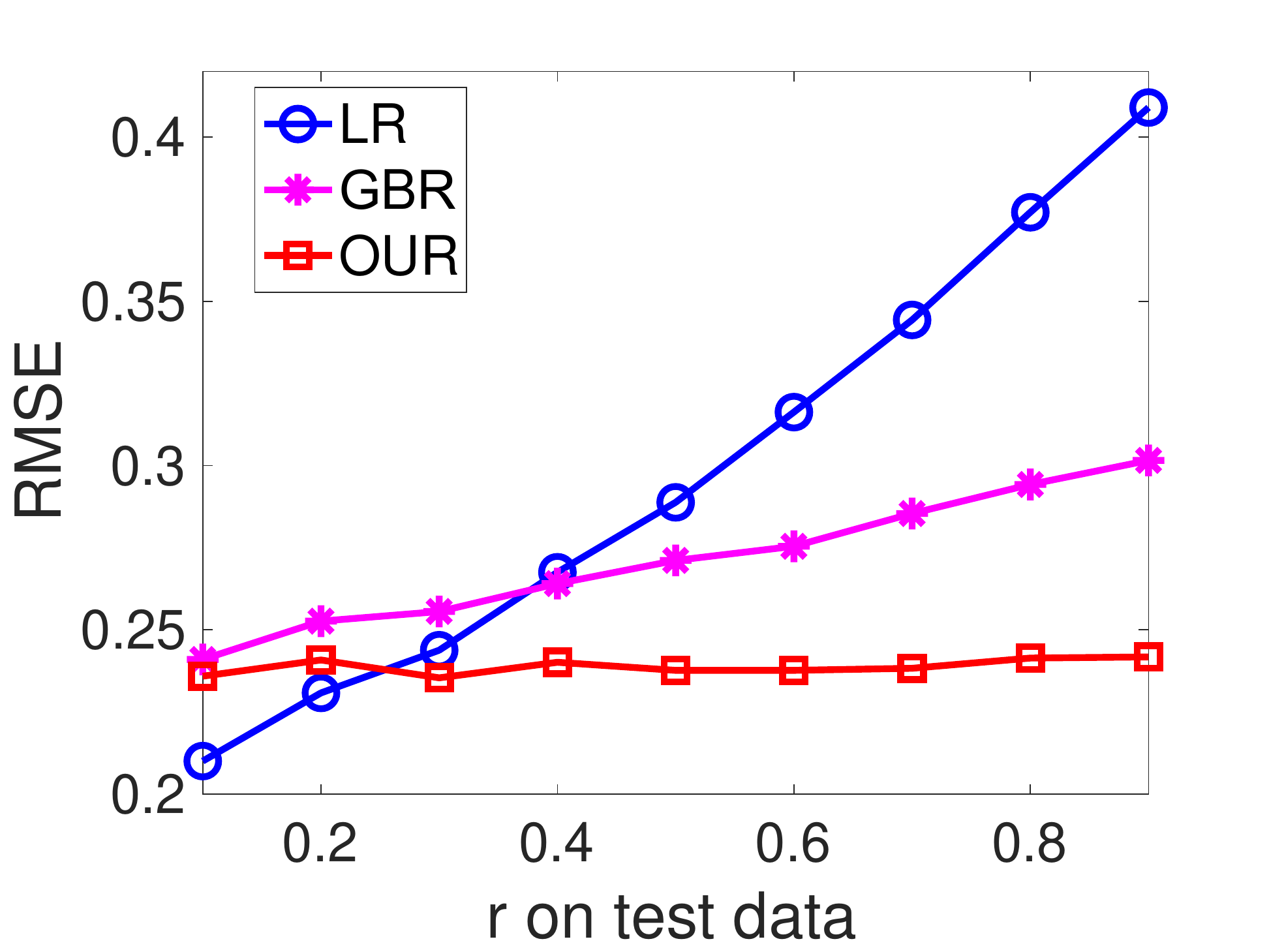}
    }
    \subfloat[\scriptsize{$r_{\rm train}=0.25$}\label{fig:C_RMSE_128-025}]{
      \includegraphics[width=1.5in]{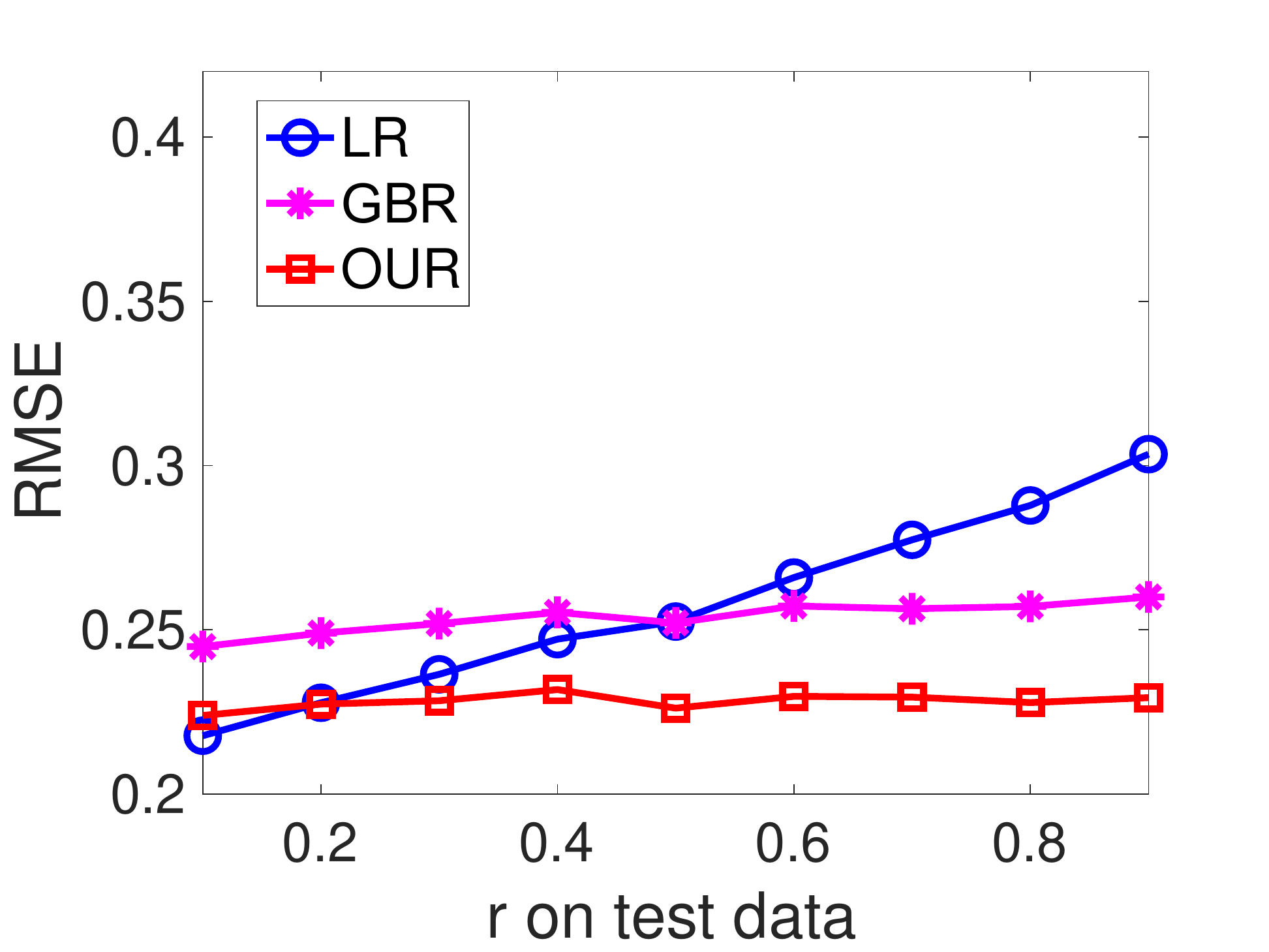}
    }
    \subfloat[\scriptsize{$r_{\rm train}=0.75$}\label{fig:C_RMSE_128-075}]{
      \includegraphics[width=1.5in]{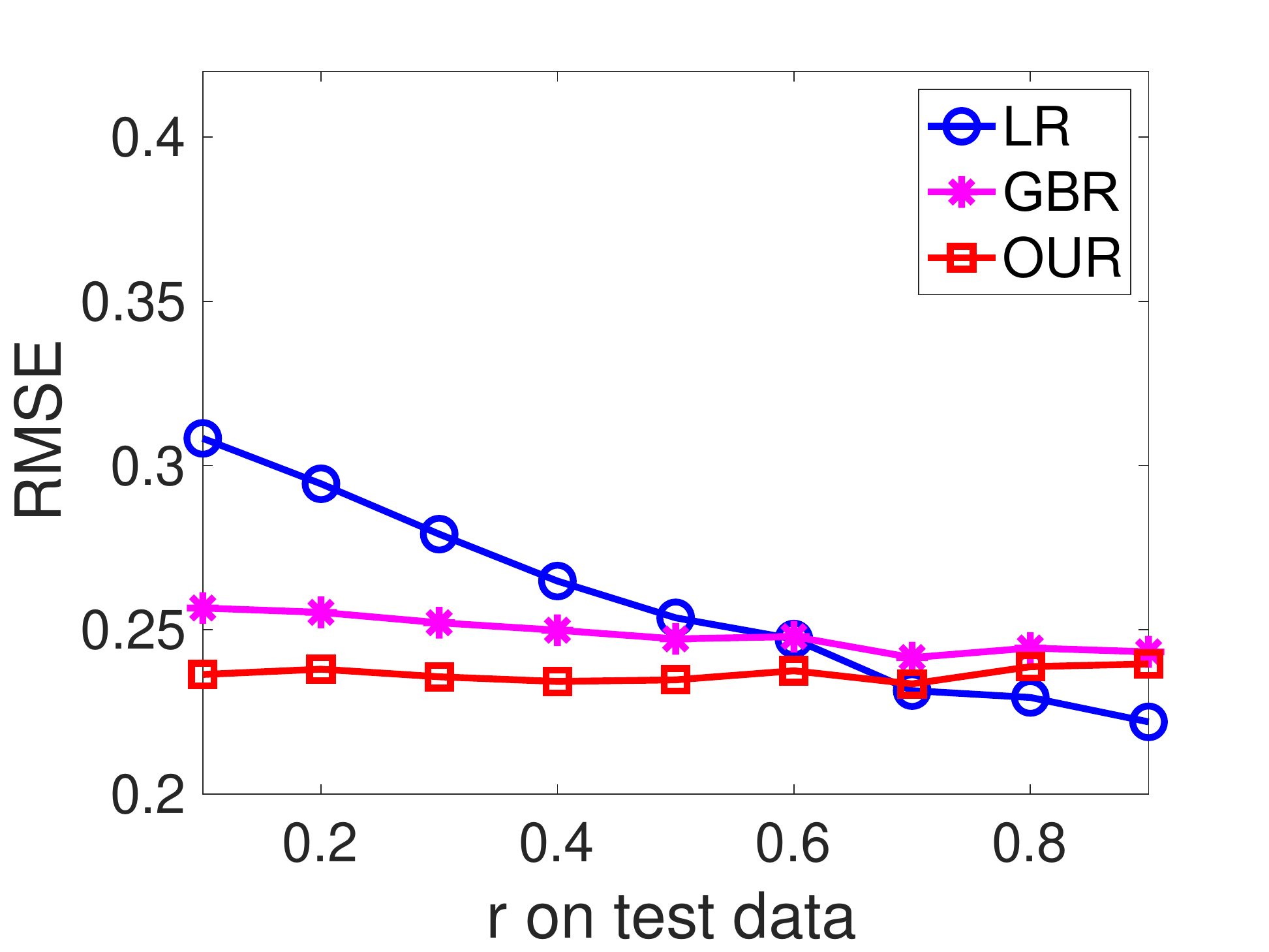}
    }
    \subfloat[\scriptsize{$r_{\rm train}=0.85$}\label{fig:C_RMSE_128-085}]{
      \includegraphics[width=1.5in]{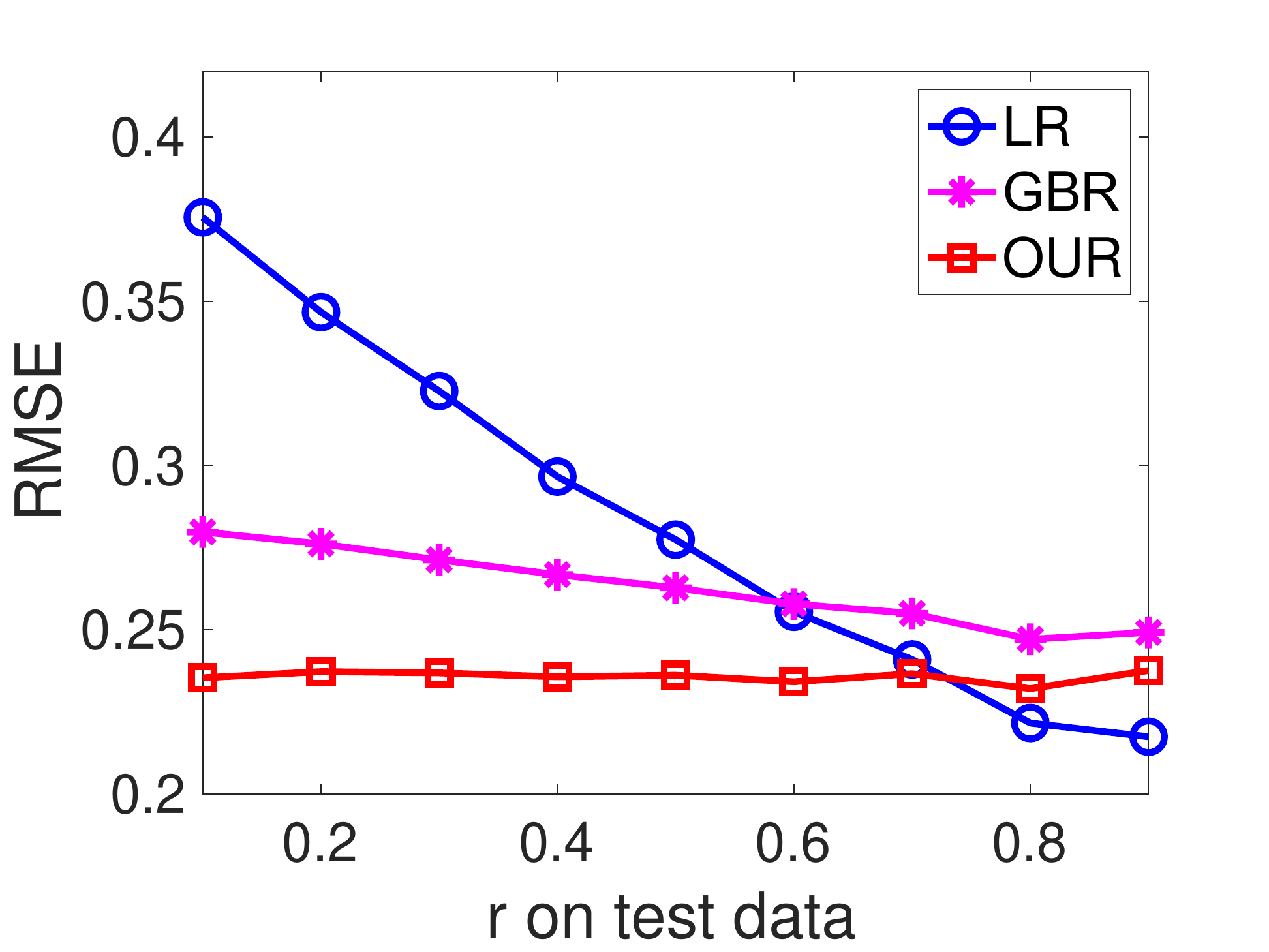}
    }
    \vspace{-0.05in}
    \caption{Results of classification on various test datasets by varying $r_{\rm train}$.}
    \label{fig:simulation_classfication}
    \vspace{-0.1in}
    \end{figure*}
    
    \subsubsection{Stable Prediction for Classification Task}
    \paragraph{Datasets.}
    The covariates are simulated from the same process in the regression task. But we generate binary response variable $Y$ from the function as follows:
    \begin{align}
    \label{data_X}
    \nonumber 
    \Scale[1.0]{Y=1/(1+\exp(-\sum_{i=1}^3\alpha_i\cdot \mathbf{S}_{\cdot,i}
    -5\cdot \mathbf{S}_{\cdot,4}\cdot \mathbf{S}_{\cdot,5})) + \mathcal{N}(0,0.2^2)}.
    \end{align}
    We define $\alpha_i = (-1)^{i}\cdot (\mbox{mod}(i,3)+1)\cdot d/3$, where function $\mbox{mod}(x,y)$
    returns the modulus after the division of $x$ by $y$. To make $Y$ binary, we set $Y=1$ when $Y\geq0.5$, otherwise $Y=0$. Furthermore, different environments are generated by varying $P(Y|\mathbf{V}_5)$ with a biased sampling rate $r\in(0,1)$. Specifically, we select a sample with probability $r$ if $\mathbf{V}_{5}=Y$; otherwise, we select it with probability $1-r$, where $r>0.5$ corresponds to a positive correlation between $Y$ and $\mathbf{V}$.
    And larger $r$ leads to strongr correlation. $n=2000$ samples are generated after selection.
    
    \vspace{-0.35cm}
    \paragraph{Results.}
    In our experiments, we generate different synthetic data by varying bias rate $r_{\rm train} \in \{0.15,0.25,0.75,0.85\}$.  For each $r_{\rm train}$, we set $r_{\rm test}\in\{0.1,0.2,\dots,1\}$. The averaged RMSE of 50 replications on each training data is visualized in Fig. \ref{fig:simulation_classfication}.  As we can see, the suspicious correlation between $\mathbf{V}_5$ and $Y$ can improve the performance of baselines when $r_{\rm train}\approx r_{\rm test}$, but it causes instability of prediction to LR or GBR when the training and test environments are quite different. 
    Compared with baselines, the RMSE of our BSSP algorithm is consistently stable and small across different environments. 
    Moreover, BSSP makes greater improvements when $r$ is farther from $0.5$. 
    
    \subsubsection{Pearson Correlation Coefficients}
          \begin{figure*}[tb]
          \centering
          \vspace{-0.1in}
          \subfloat[On raw data \label{fig:R_Corr_OLS}]{
            \includegraphics[width=1.9in]{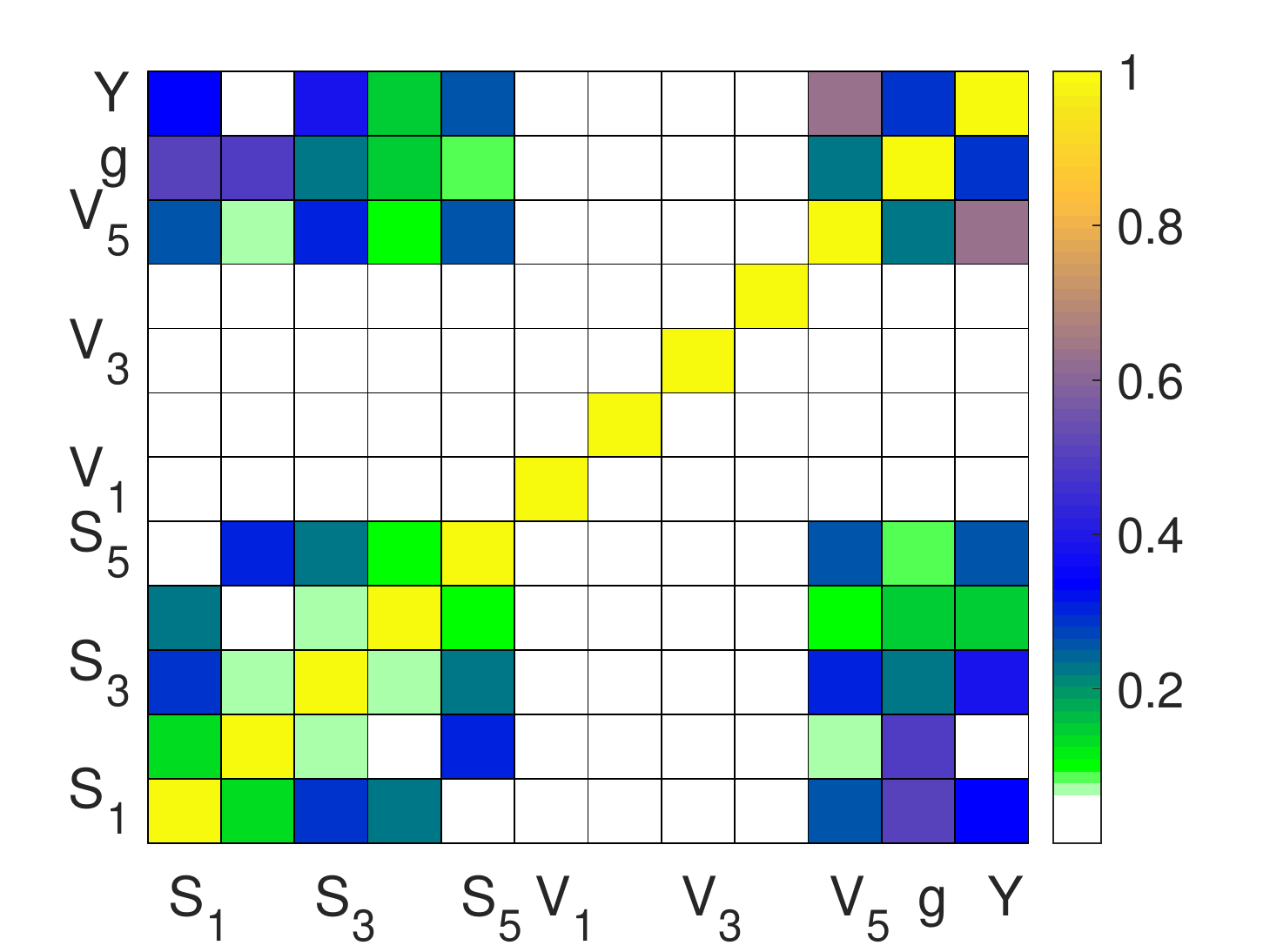}
          }
          \ \ 
          \subfloat[On the weighted data from GBR\label{fig:R_Corr_BAL}]{
            \includegraphics[width=1.9in]{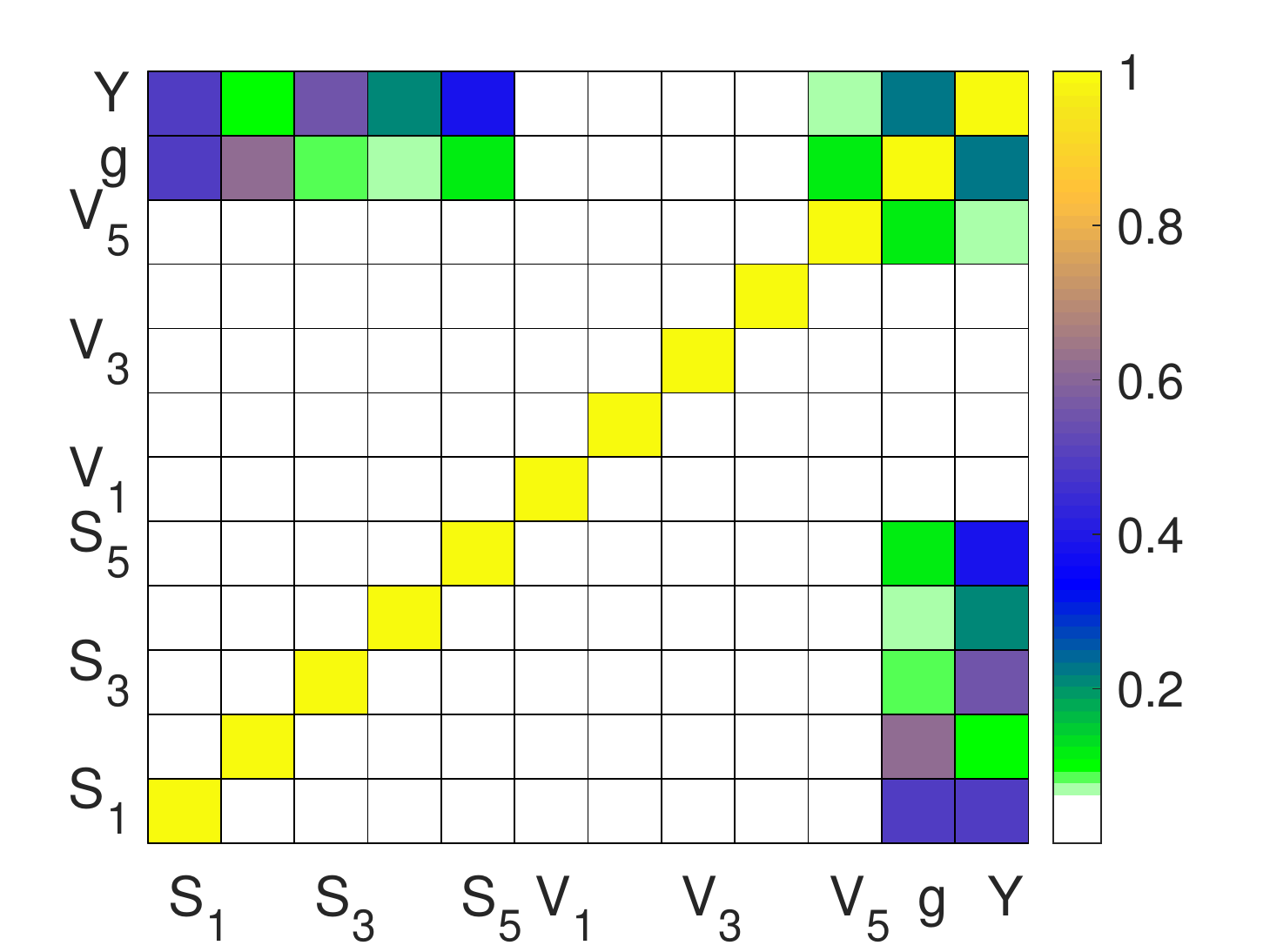}
          }
          \ \ 
          \subfloat[On subsampled data from BSSP\label{fig:R_Corr_SUB}]{
            \includegraphics[width=1.9in]{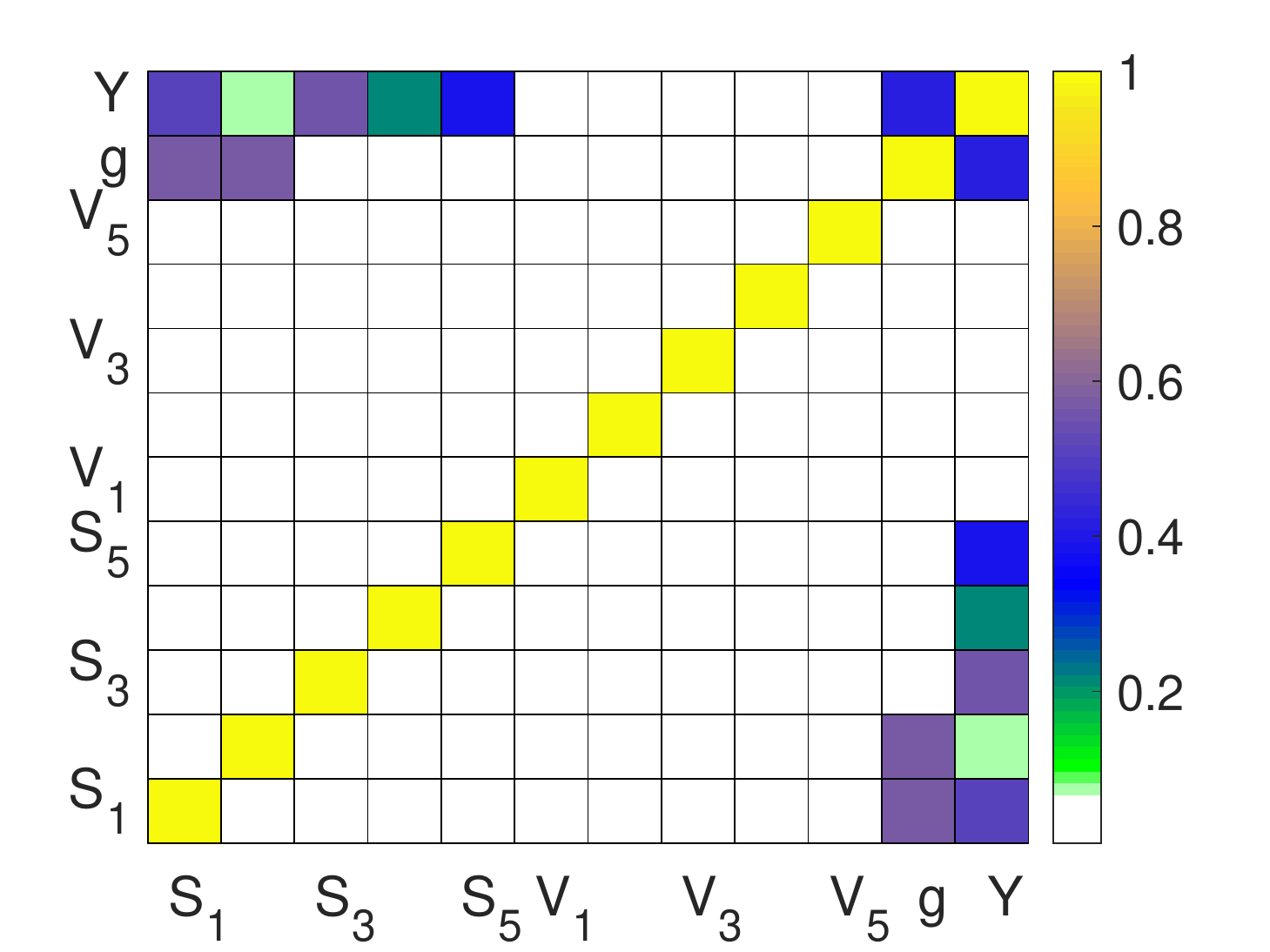}
          }
          \vspace{-0.1in}
          \caption{Pearson correlation coefficients among variables: a) on raw data; b) on weighted data from global balancing; c) on subsampled data from BSSP.}
          \label{fig:correlation}
          \vspace{-0.1in}
          \end{figure*}
          
          Under the synthetic regression setting, we compare the Pearson correlation coefficients between any two variables of various approaches in Fig. \ref{fig:correlation}. This experiment demonstrates how different methods remove the confounding effects among variables. 
          
          From the result, we can find that in the raw data (Fig. \ref{fig:R_Corr_OLS}), the noisy feature $\mathbf{V}_5$ is correlated with some stable features $\mathbf{S}$, the nonlinear term $g$ and outcome $Y$.
          Hence, the estimated coefficient of $\mathbf{V}_5$ in OLS would be large and thus leads to unstable prediction.
          In the weighted data by GBR (Fig. \ref{fig:R_Corr_BAL}), the sample weights learned from global balancing can clearly remove the correlation between noisy feature $\mathbf{V}_5$ and stable features $\mathbf{S}$. But $\mathbf{V}_5$ is still correlated with both omitted nonlinear term $g$ and outcome $Y$, leading to imprecise estimation on $\beta_V$ and unstable prediction. The main reason is that the global balancing method only considers first-order confounding effects while ignoring the higher-order effect between $\mathbf{V}_5$ and $g$.
          In BSSP (Fig. \ref{fig:R_Corr_SUB}), we can find that not only the correlation among predictors $\mathbf{X}$, but also the one between $\mathbf{V}_5$ and $g$ as well as $Y$ is clearly removed. Hence, our BSSP algorithm can estimate the coefficient of both $\mathbf{S}$ and $\mathbf{V}$ more precisely. 
          This is the key reason that BSSP can make more stable predictions across unknown test environments.
          
   \subsection{Real Datasets}
    To check the performance of the BSSP algorithm on the real-world data. We apply it to a WeChat advertising dataset (classification) and Parkinson's telemonitoring data (regression), respectively. The detailed descriptions of datasets and preprocessing are listed in what follows.
    \vspace{-0.35cm}
    \paragraph{WeChat Advertising Dataset.}
     It is a real online advertising dataset, which is collected from Tencent WeChat App\footnote{http://www.wechat.com/en/} during September 2015 and used in \cite{kuang2018stable} for stable prediction.
     In WeChat, each user can share (receive) posts to (from) his/her friends like Twitter and Facebook.
     Then the advertisers could push their advertisements to users, by merging them into the list of the user's wall posts.
     For each advertisement, there are two types of feedbacks: ``Like'' and ``Dislike''.
     When the user clicks the ``Like'' button, his/her friends will receive the advertisements.
     
     The WeChat advertising campaign used in our paper is about the LONGCHAMP handbags
     for young women.\footnote{http://en.longchamp.com/en/womens-bags}
     This campaign contains 14,891 user feedbacks with Like and 93,108 Dislikes.
     For each user, we have their features including (1) demographic attributes, such as age, gender,
     (2) number of friends, (3) device (iOS or Android), and (4) the user settings on WeChat, for example, whether allowing strangers
     to see his/her album and whether installing the online payment service.
     
     In our experiments, we set $Y_i=1$ if user $i$ likes the ad, otherwise $Y_i=0$. For non-binary features, we dichotomize them around their mean value. And we only preserve users' features which satisfied $0.2 \leq\frac{\#\{x=1\}}{\#\{x=1\}+\#\{x=0\}} \leq 0.8$. Finally, our dataset contains 10 binary user features as predictors and user feedback as the outcome variable. To test the stability of all methods, we separate the whole dataset into 4 parts by users' age, including $\mbox{Age} \in [20,30)$, $\mbox{Age} \in [30,40)$, $\mbox{Age} \in [40,50)$ and $\mbox{Age} \in [50,100)$. All models are trained with data from the environment $\mbox{Age} \in [20,30)$ but are tested on all 4 environments.
     
    \vspace{-0.35cm}
    \paragraph{Parkinson's Telemonitoring Dataset.}
    To test BSSP in a regression setting, we apply it to a Parkinson's telemonitoring dataset\footnote{\url{https://archive.ics.uci.edu/ml/datasets/parkinsons+telemonitoring}}, which has been wildly used for domain generalization \citep{muandet2013domain,blanchard2017domain} task and other regression task \citep{tsanas2009accurate}. The dataset is composed of biomedical voice measurements from 42 patients with early-stage Parkinson's disease. For each patient, there are around 200 recordings, which were automatically captured in the patients' homes. The aim is to predict the clinician's motor and total UPDRS scores of Parkinson's disease symptoms from patients' features, including their age, gender, test time and many other measures.
          
    In our experiments, we alternately set the outcome variables $Y$ as motor UPDRS scores and total UPDRS scores. For those non-binary features, we dichotomize them around their mean value. Finally, we selected 10 patients features as predictors $\mathbf{X}$, including age, gender, test time, Jitter:PPQ5 (a measure of variation in fundamental frequency), Shimmer:APQ5 (a measure of variation in amplitude), RPDE (a nonlinear dynamical complexity measure), DFA (signal fractal scaling exponent), PPE (a nonlinear measure of fundamental frequency variation), NHR and HNR (two measures of the ratio of noise to tonal components in the voice).
          
    To test the stability of all algorithms, we separate the whole 42 patients into 4 patients' groups, including PG1 with recordings from 21 patients, and the other three groups (PG2, PG3, and PG4) are all with recordings from different 7 patients. We train models with data from environment PG1, but test them on all 4 environments.
          
    \vspace{-0.35cm}
    \paragraph{Results.} 
    We visualize the results in Figure \ref{fig:real_data}. From the figure, we can obtain that the proposed BSSP algorithm achieves comparable results to the baseline OLS/LR on training environment. On the other three test environments, whose distributions differ from the training environment, BSSP achieves the best prediction performance.
    Another important observation is that the performance of our algorithm is always better than the global balancing method. The main reason is that GBR, unlike our BSSP algorithm, cannot address the high-order confounding among variables.
    
    \vspace{-0.1in}
    \begin{figure}[t]
    \vspace{-0.1in}
    \centering
    \subfloat[Wechat data\label{fig:real_classification}]{
      \includegraphics[width=1.9in]{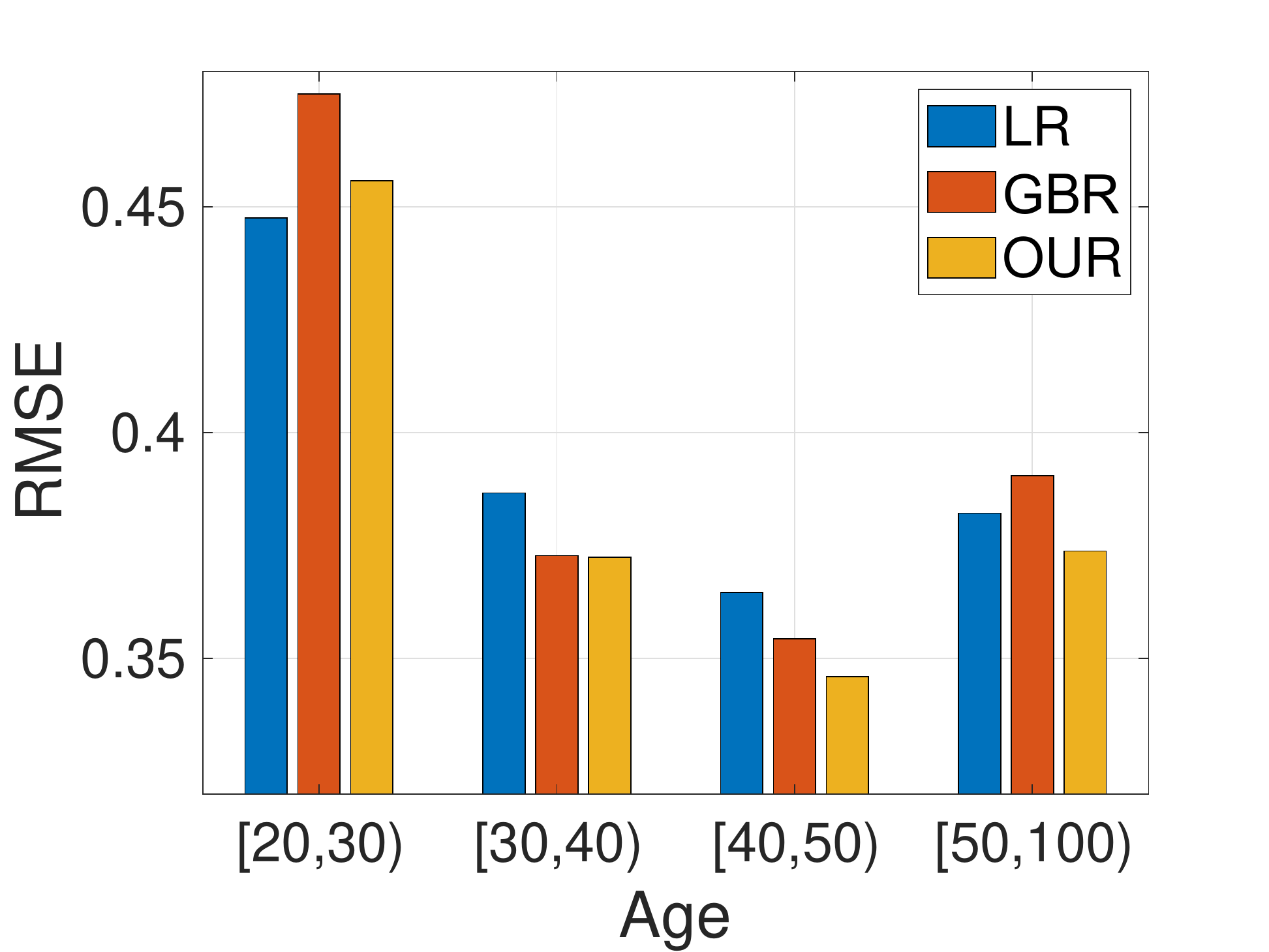}
    }
    \ \
    \subfloat[Parkinson data with $Y$ as motor UPDRS score \label{fig:real_regression_Y_motor_RMSE}]{
      \includegraphics[width=1.9in]{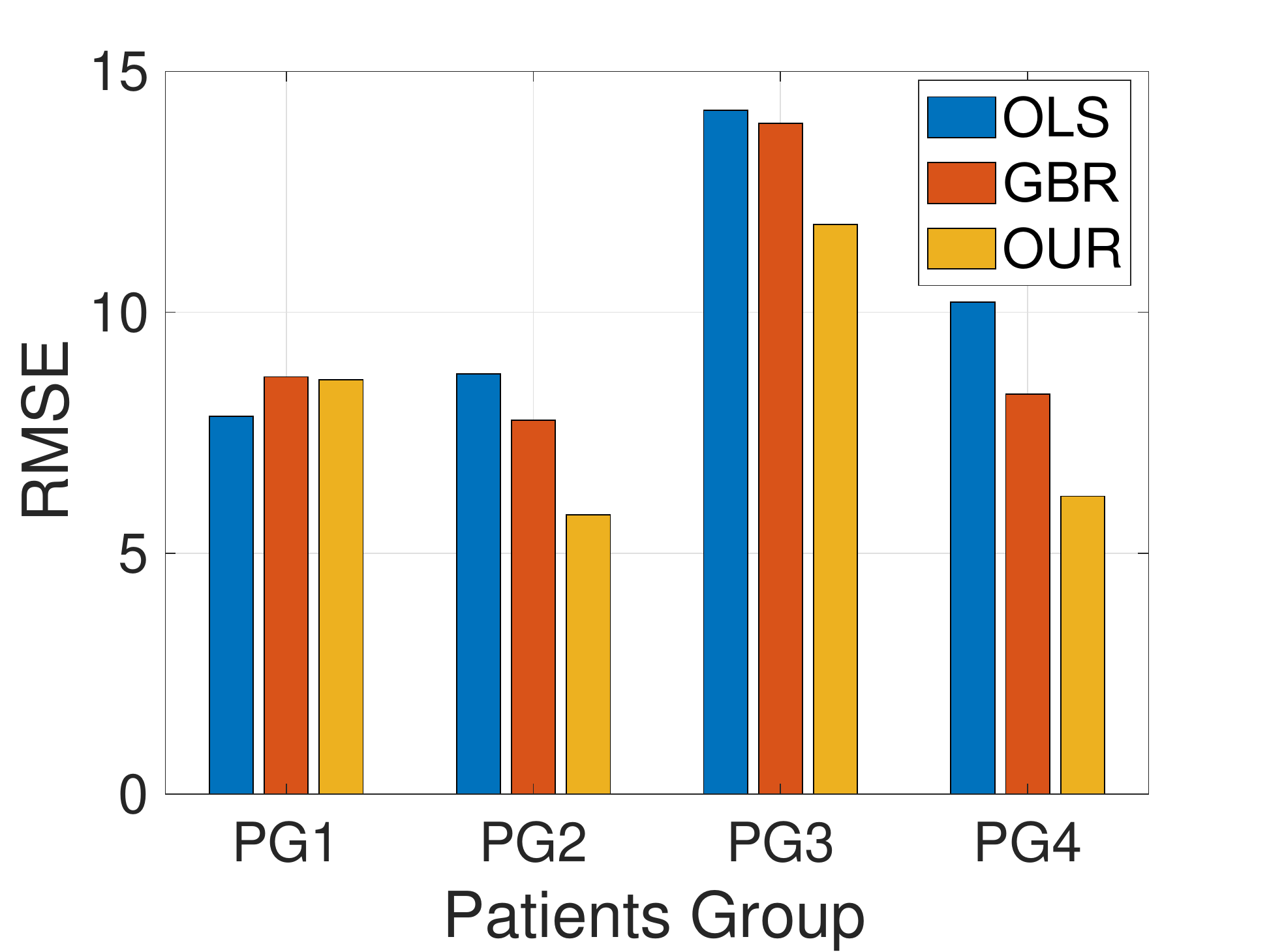}
    }
    \ \
    \subfloat[Parkinson data with $Y$ as total UPDRS score \label{fig:real_regression_Y_total_RMSE}]{
      \includegraphics[width=1.9in]{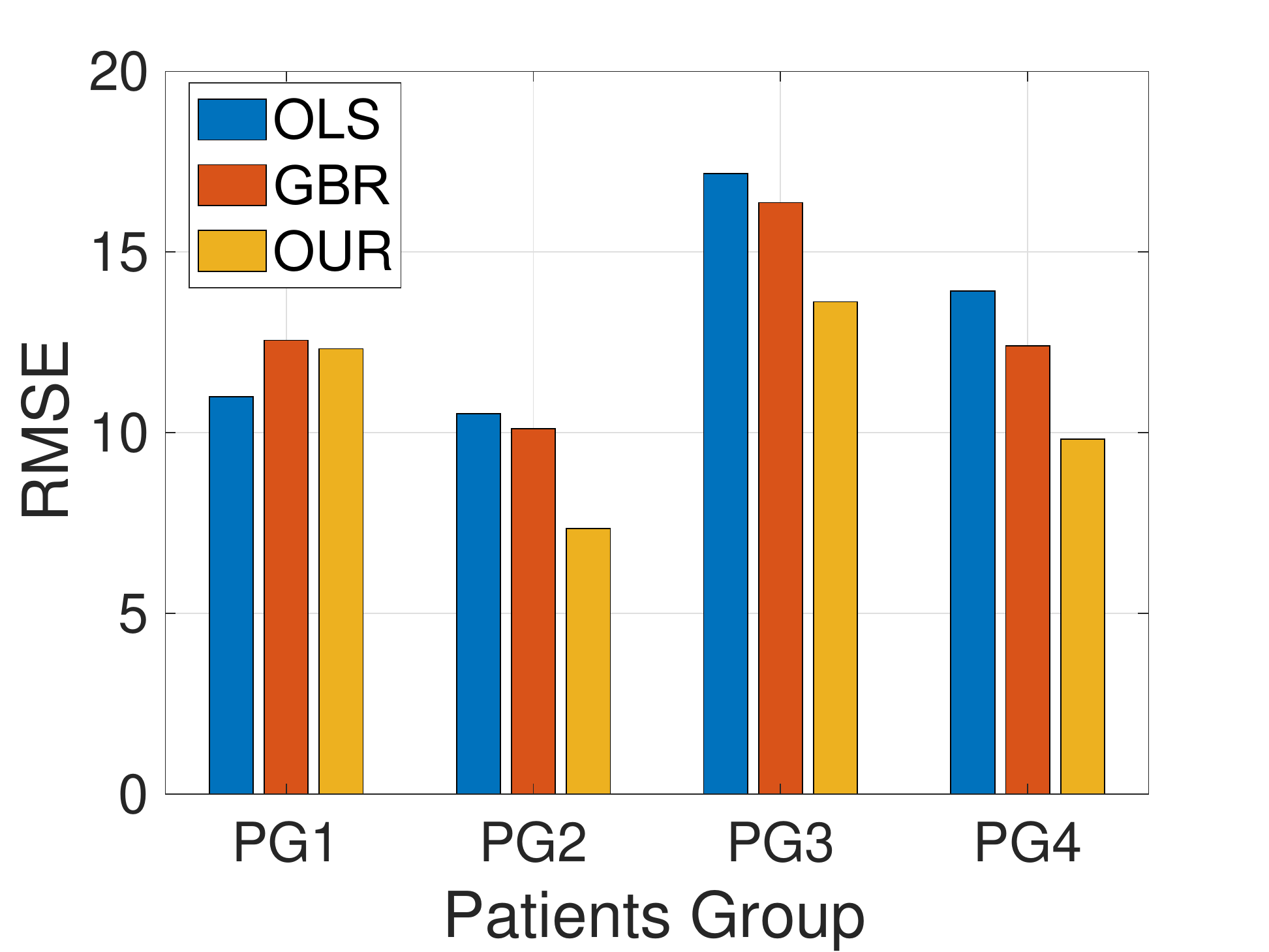}
    }
    \vspace{-0.05in}
    \caption{Prediction across environments of each method in different datasets. The models are all trained on the first environment, but tested across different groups. }
    \label{fig:real_data}
    \vspace{-0.1in}
    \end{figure}
    
    \section{Conclusion}
    This paper addresses the problem of stable prediction across unknown environments. We propose a subsampling method to reduce the spurious correlation between noisy features and the outcome variable.
    The subsampling method uses fractional factorial design as a matching template, which can promote the non-confounding properties among predictors if one can find a subdata to fully match the design. We also propose a new confounding measure to guide the subsample selection in general situations. Our method can be regarded as a data pre-treatment so that it can be applied to different prediction tasks, such as regression and classification. Extensive experiments on both synthetic and real-world datasets have clearly demonstrated the advantages of our proposed method for stable prediction.
    Our future work will focus on the subsampling based on $k$-level fractional factorial designs and space-filling designs, aiming to address the stable prediction problem with multi-level and continuous variables.

    \clearpage
    \bibliographystyle{plainnat}
    \bibliography{stable-prediction}

\begin{thebibliography}{30}
\providecommand{\natexlab}[1]{#1}
\providecommand{\url}[1]{\texttt{#1}}
\expandafter\ifx\csname urlstyle\endcsname\relax
  \providecommand{\doi}[1]{doi: #1}\else
  \providecommand{\doi}{doi: \begingroup \urlstyle{rm}\Url}\fi

\bibitem[Athey et~al.(2018)Athey, Imbens, and Wager]{athey2018approximate}
Susan Athey, Guido~W Imbens, and Stefan Wager.
\newblock Approximate residual balancing: debiased inference of average
  treatment effects in high dimensions.
\newblock \emph{Journal of the Royal Statistical Society: Series B (Statistical
  Methodology)}, 80\penalty0 (4):\penalty0 597--623, 2018.

\bibitem[Austin(2011)]{austin2011introduction}
Peter~C Austin.
\newblock An introduction to propensity score methods for reducing the effects
  of confounding in observational studies.
\newblock \emph{Multivariate Behavioral Research}, 46\penalty0 (3):\penalty0
  399--424, 2011.

\bibitem[Blanchard et~al.(2017)Blanchard, Deshmukh, Dogan, Lee, and
  Scott]{blanchard2017domain}
Gilles Blanchard, Aniket~Anand Deshmukh, Urun Dogan, Gyemin Lee, and Clayton
  Scott.
\newblock Domain generalization by marginal transfer learning.
\newblock \emph{arXiv preprint arXiv:1711.07910}, 2017.

\bibitem[Box et~al.(2005)Box, Hunter, and Hunter]{box2005statistics}
George~EP Box, J~Stuart Hunter, and William~G Hunter.
\newblock Statistics for experimenters.
\newblock In \emph{Wiley Series in Probability and Statistics}. Wiley Hoboken,
  NJ, USA, 2005.

\bibitem[Dey and Mukerjee(2009)]{dey2009fractional}
Aloke Dey and Rahul Mukerjee.
\newblock \emph{Fractional Factorial Plans}, volume 496.
\newblock John Wiley \& Sons, 2009.

\bibitem[Drineas et~al.(2011)Drineas, Mahoney, Muthukrishnan, and
  Sarl{\'o}s]{drineas2011faster}
Petros Drineas, Michael~W Mahoney, Shan Muthukrishnan, and Tam{\'a}s
  Sarl{\'o}s.
\newblock Faster least squares approximation.
\newblock \emph{Numerische Mathematik}, 117\penalty0 (2):\penalty0 219--249,
  2011.

\bibitem[Fries and Hunter(1980)]{fries1980minimum}
Arthur Fries and William~G Hunter.
\newblock Minimum aberration $2^{k-p}$ designs.
\newblock \emph{Technometrics}, 22\penalty0 (4):\penalty0 601--608, 1980.

\bibitem[Gr{\"o}nmping(2014)]{gronmping2014r}
Ulrike Gr{\"o}nmping.
\newblock R package frf2 for creating and analyzing fractional factorial
  2-level designs.
\newblock \emph{Journal of Statistical Software}, 56\penalty0 (1):\penalty0
  1--56, 2014.

\bibitem[Hainmueller(2012)]{hainmueller2012entropy}
Jens Hainmueller.
\newblock Entropy balancing for causal effects: A multivariate reweighting
  method to produce balanced samples in observational studies.
\newblock \emph{Political Analysis}, 20\penalty0 (1):\penalty0 25--46, 2012.

\bibitem[Hedayat et~al.(2012)Hedayat, Sloane, and
  Stufken]{hedayat2012orthogonal}
A~Samad Hedayat, Neil James~Alexander Sloane, and John Stufken.
\newblock \emph{Orthogonal Arrays: Theory and Applications}.
\newblock Springer Science \& Business Media, 2012.

\bibitem[Kuang et~al.(2017{\natexlab{a}})Kuang, Cui, Li, Jiang, and
  Yang]{kuang2017estimating}
Kun Kuang, Peng Cui, Bo~Li, Meng Jiang, and Shiqiang Yang.
\newblock Estimating treatment effect in the wild via differentiated confounder
  balancing.
\newblock In \emph{Proceedings of the 23rd ACM SIGKDD International Conference
  on Knowledge Discovery and Data Mining}, pages 265--274. ACM,
  2017{\natexlab{a}}.

\bibitem[Kuang et~al.(2017{\natexlab{b}})Kuang, Cui, Li, Jiang, Yang, and
  Wang]{kuang2017treatment}
Kun Kuang, Peng Cui, Bo~Li, Meng Jiang, Shiqiang Yang, and Fei Wang.
\newblock Treatment effect estimation with data-driven variable decomposition.
\newblock In \emph{AAAI}, pages 140--146, 2017{\natexlab{b}}.

\bibitem[Kuang et~al.(2018)Kuang, Cui, Athey, Xiong, and Li]{kuang2018stable}
Kun Kuang, Peng Cui, Susan Athey, Ruoxuan Xiong, and Bo~Li.
\newblock Stable prediction across unknown environments.
\newblock In \emph{Proceedings of the 24th ACM SIGKDD International Conference
  on Knowledge Discovery \& Data Mining}, pages 1617--1626. ACM, 2018.

\bibitem[Ma and Fang(2001)]{ma2001note}
Chang-Xing Ma and Kai-Tai Fang.
\newblock A note on generalized aberration in factorial designs.
\newblock \emph{Metrika}, 53\penalty0 (1):\penalty0 85--93, 2001.

\bibitem[Ma et~al.(2015)Ma, Mahoney, and Yu]{ma2015statistical}
Ping Ma, Michael~W Mahoney, and Bin Yu.
\newblock A statistical perspective on algorithmic leveraging.
\newblock \emph{The Journal of Machine Learning Research}, 16\penalty0
  (1):\penalty0 861--911, 2015.

\bibitem[MacWilliams and Sloane(1977)]{macwilliams1977theory}
Florence~Jessie MacWilliams and Neil James~Alexander Sloane.
\newblock \emph{The Theory of Error-Correcting Codes}, volume~16.
\newblock Elsevier, 1977.

\bibitem[Muandet et~al.(2013)Muandet, Balduzzi, and
  Sch{\"o}lkopf]{muandet2013domain}
Krikamol Muandet, David Balduzzi, and Bernhard Sch{\"o}lkopf.
\newblock Domain generalization via invariant feature representation.
\newblock In \emph{Proceedings of the 30th International Conference on Machine
  Learning}, pages 10--18, 2013.

\bibitem[Peters et~al.(2016)Peters, B{\"u}hlmann, and
  Meinshausen]{peters2016causal}
Jonas Peters, Peter B{\"u}hlmann, and Nicolai Meinshausen.
\newblock Causal inference by using invariant prediction: identification and
  confidence intervals.
\newblock \emph{Journal of the Royal Statistical Society: Series B (Statistical
  Methodology)}, 78\penalty0 (5):\penalty0 947--1012, 2016.

\bibitem[Rojas-Carulla et~al.(2018)Rojas-Carulla, Sch{\"o}lkopf, Turner, and
  Peters]{rojas2018invariant}
Mateo Rojas-Carulla, Bernhard Sch{\"o}lkopf, Richard Turner, and Jonas Peters.
\newblock Invariant models for causal transfer learning.
\newblock \emph{The Journal of Machine Learning Research}, 19\penalty0
  (1):\penalty0 1309--1342, 2018.

\bibitem[Rosenbaum and Rubin(1983)]{rosenbaum1983central}
Paul~R Rosenbaum and Donald~B Rubin.
\newblock The central role of the propensity score in observational studies for
  causal effects.
\newblock \emph{Biometrika}, 70\penalty0 (1):\penalty0 41--55, 1983.

\bibitem[Thompson(2012)]{steven2012sampling}
Steven~K. Thompson.
\newblock \emph{Sampling}.
\newblock Wiley Series in Probability and Statistics. Wiley, 3 edition, 2012.
\newblock ISBN 0470402318,9780470402313.

\bibitem[Tibshirani(1996)]{tibshirani1996regression}
Robert Tibshirani.
\newblock Regression shrinkage and selection via the lasso.
\newblock \emph{Journal of the Royal Statistical Society: Series B
  (Methodological)}, 58\penalty0 (1):\penalty0 267--288, 1996.

\bibitem[Tsanas et~al.(2009)Tsanas, Little, McSharry, and
  Ramig]{tsanas2009accurate}
Athanasios Tsanas, Max~A Little, Patrick~E McSharry, and Lorraine~O Ramig.
\newblock Accurate telemonitoring of parkinson's disease progression by
  noninvasive speech tests.
\newblock \emph{IEEE Transactions on Biomedical Engineering}, 57\penalty0
  (4):\penalty0 884--893, 2009.

\bibitem[Wang(2019)]{wang2019more}
HaiYing Wang.
\newblock More efficient estimation for logistic regression with optimal
  subsamples.
\newblock \emph{The Journal of Machine Learning Research}, 20\penalty0
  (132):\penalty0 1--59, 2019.

\bibitem[Wang et~al.(2018)Wang, Zhu, and Ma]{wang2018optimal}
HaiYing Wang, Rong Zhu, and Ping Ma.
\newblock Optimal subsampling for large sample logistic regression.
\newblock \emph{Journal of the American Statistical Association}, 113\penalty0
  (522):\penalty0 829--844, 2018.

\bibitem[Wang et~al.(2019)Wang, Yang, and Stufken]{wang2019information}
HaiYing Wang, Min Yang, and John Stufken.
\newblock Information-based optimal subdata selection for big data linear
  regression.
\newblock \emph{Journal of the American Statistical Association}, 114\penalty0
  (525):\penalty0 393--405, 2019.

\bibitem[Wu and Hamada(2011)]{wu2011experiments}
CF~Jeff Wu and Michael~S Hamada.
\newblock \emph{Experiments: Planning, Analysis, and Optimization}, volume 552.
\newblock John Wiley \& Sons, 2011.

\bibitem[Xu and Wu(2001)]{xu2001generalized}
Hongquan Xu and CF~Jeff Wu.
\newblock Generalized minimum aberration for asymmetrical fractional factorial
  designs.
\newblock \emph{The Annals of Statistics}, 29\penalty0 (4):\penalty0
  1066--1077, 2001.

\bibitem[Zhang et~al.(2005)Zhang, Fang, Li, Sudjianto,
  et~al.]{zhang2005majorization}
Aijun Zhang, Kai-Tai Fang, Runze Li, Agus Sudjianto, et~al.
\newblock Majorization framework for balanced lattice designs.
\newblock \emph{The Annals of Statistics}, 33\penalty0 (6):\penalty0
  2837--2853, 2005.

\bibitem[Zubizarreta(2015)]{zubizarreta2015stable}
Jos{\'e}~R Zubizarreta.
\newblock Stable weights that balance covariates for estimation with incomplete
  outcome data.
\newblock \emph{Journal of the American Statistical Association}, 110\penalty0
  (511):\penalty0 910--922, 2015.

\end{thebibliography}

\end{document}